\documentclass[journal]{IEEEtran}
\usepackage{hyperref}
\ifCLASSINFOpdf
\else
\fi

\usepackage{amssymb}
\usepackage{amsmath} %
\usepackage{mathtools} %
\usepackage{latexsym}
\usepackage{amsthm}%
\usepackage{bbm}%
\usepackage[space]{grffile} %

\usepackage{wrapfig}
\usepackage{caption}
\usepackage{subcaption}

\makeatletter
\newtheorem*{rep@theorem}{\rep@title}
\newcommand{\newreptheorem}[2]{%
	\newenvironment{rep#1}[1]{%
		\def\rep@title{#2 \ref{##1}}%
		\begin{rep@theorem}}%
		{\end{rep@theorem}}}
\makeatother
 \newreptheorem{theorem}{Theorem}
 \newreptheorem{lemma}{Lemma}

 \newtheorem{theorem}{Theorem}
 
 \newtheorem{lemma}[theorem]{Lemma}
 \newtheorem{remark}{Remark}

 \theoremstyle{definition}

\makeatletter
\newcommand{\distas}[1]{\mathbin{\overset{#1}{\kern\z@\sim}}}%
\newsavebox{\mybox}\newsavebox{\mysim}
\newcommand{\distras}[1]{%
	\savebox{\mybox}{\hbox{\kern3pt$\scriptstyle#1$\kern3pt}}%
	\savebox{\mysim}{\hbox{$\sim$}}%
	\mathbin{\overset{#1}{\kern\z@\resizebox{\wd\mybox}{\ht\mysim}{$\sim$}}}%
}
\makeatother

\DeclareMathOperator{\EX}{\mathbb{E}}%
\newcommand{\norm}[1]{\left\lVert#1\right\rVert}

\DeclareMathOperator*{\argmin}{arg\,min}
\newcommand\independent{\protect\mathpalette{\protect\independenT}{\perp}}
\def\independenT#1#2{\mathrel{\rlap{$#1#2$}\mkern2mu{#1#2}}}

\DeclarePairedDelimiterX{\infdivx}[2]{(}{)}{%
	#1\;\delimsize\|\;#2%
}
\newcommand{\KLdiv}{\mathrm{KL}\infdivx}
\newcommand{\DKL}{\mathrm{D}\infdivx}
\newcommand{\cW}{\mathcal{W}}
\newcommand{\cA}{\mathcal{A}}
\newcommand{\alg}{A}
\newcommand{\cX}{\mathcal{X}}
\newcommand{\cY}{\mathcal{Y}}
\newcommand{\dvc}{D} %

\usepackage{tikz}
\usepackage{textcomp}
\usepackage{hyperref}
\usepackage{lipsum}

\makeatletter
\def\ps@IEEEtitlepagestyle{%
  \def\@oddfoot{\mycopyrightnotice}%
  \def\@oddhead{\hbox{}\@IEEEheaderstyle\leftmark\hfil\thepage}\relax
  \def\@evenhead{\@IEEEheaderstyle\thepage\hfil\leftmark\hbox{}}\relax
  \def\@evenfoot{}%
}
\def\mycopyrightnotice{%
\fbox{  \begin{minipage}{\textwidth}
  \centering \scriptsize
  Copyright~\copyright~2023 IEEE. Personal use of this material is permitted. However, permission to use this material for any other purposes must be obtained from the IEEE by sending a request to pubs-permissions@ieee.org.
  \end{minipage}
  }
}
\makeatother

\begin{document}
\title{Information-Theoretic Analysis\\ of Minimax Excess Risk}

\author{Hassan Hafez-Kolahi$^*$,
        Behrad Moniri$^*$,
        and~Shohreh Kasaei,~\IEEEmembership{Senior~Member,~IEEE}%

\thanks{Accepted by the IEEE Transactions on Information Theory for publication (DOI: 10.1109/TIT.2023.3249636) and the final version is available on IEEE Xplore. Manuscript received February 14, 2022; revised September 30, 2022; accepted February 8, 2023. \textit{($^*$Hassan Hafez-Kolahi and Behrad Moniri contributed equally to this work.) }}
\thanks{Hassan Hafez-Kolahi and Shohreh Kasaei are with the Image Processing Laboratory (IPL), Department of Computer Engineering, Sharif University of Technology, Tehran, Iran. Behrad Moniri is  with the Department of Electrical and Systems Engineering, University of Pennsylvania, Philadelphia, PA, USA.   (e-mail: hafez@ce.sharif.edu, bemoniri@seas.upenn.edu, kasaei@sharif.edu).}
}

\markboth{IEEE TRANSACTIONS ON INFORMATION THEORY}%
{IEEE TRANSACTIONS ON INFORMATION THEORY}

\maketitle

\begin{abstract}
Two main concepts studied in machine learning theory are generalization gap (difference between train and test error) and excess risk (difference between test error and the minimum possible error). While information-theoretic tools have been used extensively to study the generalization gap of learning algorithms, the information-theoretic nature of excess risk has not yet been fully investigated. In this paper, some steps are taken toward this goal. We consider the frequentist  problem of minimax excess risk as a zero-sum game between the algorithm designer and the world. Then, we argue that it is desirable to modify this game in a way that the order of play can be swapped. We then prove that, under some regularity conditions, if the world and designer can play randomly the duality gap is zero and the order of play can be changed. In this case, a Bayesian problem surfaces in the dual representation. This makes it possible to utilize recent information-theoretic results on minimum excess risk in Bayesian learning to provide bounds on the minimax excess risk. We demonstrate the applicability of the results by providing information theoretic insight on two important classes of problems: classification when the hypothesis space has finite VC-dimension, and regularized least squares.
\end{abstract}

\begin{IEEEkeywords}
excess risk, information-theoretic bounds, minimax theorems, frequentist-Bayesian duality.
\end{IEEEkeywords}

\section{Introduction}
\IEEEPARstart{O}{ne} of the main problems in theoretical machine learning is the study of excess risk, defined as the difference between the test error and the best achievable error if the distribution was known. There is a rich body of literature on analyzing the excess risk of specific algorithms as well as existence results on the minimum possible excess risk \cite{mohri}.
Classically, uniform bounds were used to study the excess risk of Empirical Risk Minimization (ERM) algorithms under the assumption that the best possible hypothesis is from a hypothesis class with finite VC-dimension. In this case, it can be shown that controlling excess risk is equivalent to uniformly controlling the generalization gap of all hypotheses in the class \cite{mohri}.

It is not always possible to analyze the generalization gap of learning algorithms via uniform convergence (e.g., see \cite{ZhangUnderstandingdeeplearning2017} which demonstrates the limitations of such methods on studying deep neural networks).
There are a variety of ways to derive non-uniform generalization bounds. 
Among these approaches, using information-theoretic tools is one of the main methods which has gained much attention in recent years \cite{RussoHowmuchdoes2015, XuInformationtheoreticanalysisgeneralization2017,AsadiChainingMutualInformation2018a,AsadiChainingMeetsChain2019,negrea2019information,steinke_reasoning_2020,hafez2020conditioning,haghifam_sharpened_2020,hellstrom2020nonvacuous,hellstrom2020generalization}.
In this approach, usually, a form of mutual information between the learned hypothesis and the dataset appears in the bound on the generalization gap. 
Intuitively,
this  means that a learning algorithm that has a small dependence on the dataset can not overfit to the training set beyond a certain degree. It is shown that many of the classic bounds on the generalization gap of learning algorithms can be derived using such information-theoretic tools (e.g., see \cite{steinke_reasoning_2020}). Moreover, this approach takes into account the details of the learning algorithm and its relation with the data generating distribution, which can potentially lead to much better bounds compared to uniform bounds \cite{AsadiChainingMutualInformation2018a,bu_tightening_2019}.

Most of the recent studies in this track focused on analyzing the generalization gap and not the excess risk. While there is a direct relationship between the generalization gap and the excess risk in ERM, the situation is not that simple when dealing with non-ERM algorithms and non-uniform bounds (see Section~\ref{sec:gengap_excess_NFLT}). 

Recently, information-theoretic tools have been utilized to study the minimum excess risk in the Bayesian learning \cite{xu2020minimum,xu2020continuity, hafez2021rate}.
In this setting, a variable $W$ is considered to define the distribution of data, and a prior distribution $P_W$ is considered on $W$. Minimum excess risk in Bayesian learning is defined as the \emph{expected} difference between the minimum risk of learning from data compared to the best risk achievable when the parameter $W$ is known. These recent results
extend the classic results of sequential prediction \cite{merhav1998universal,davisson1973universal,haussler1997mutual} to supervised setting with general subgaussian loss. They also utilize a refined treatment which yields bounds on the error of the last sample (test sample) when the training set is given (in sequential prediction, the accumulated loss on a sequence of samples is studied using universal source coding tools).

It should be noted that in the Bayesian setting, the distributions are assumed to be known, which makes the setting different from the usual minimax approach of the frequentist learning in which the worst distribution is considered when providing bounds.
As such, the frequentist approach is considered to be much more challenging \cite{efron1998ra}. In this paper, we
propose a framework which provides an information theoretic insight for the minimax setting as well.
One of the main contributions of the paper is to 
provide conditions under which the minimax game of excess risk has a value; i.e., the order of players can be swapped and the optimal gained value remains the same for both of them. This results in an equal dual problem which has a Bayesian flavor. 
This allows us to use the recent information-theoretic bounds on minimum excess risk in Bayesian setting to study the problem in the minimax setting.
Moreover, we propose an algorithm which can asymptotically achieve this optimal minimax value (in an intermediate setting between Bayesian and frequentist settings).

Finally, some applications are presented demonstrating the generality of this framework to study learning problems.
In particular, we focus on two 
of the most studied settings in machine learning;
i) classification when the hypothesis space has finite VC-dimension
and ii) regression with (kernel) regularized least squares.
It is demonstrated that the proposed framework provides simple information theoretic 
insight on
the learning in these setting (which is novel to the best of our knowledge).

The proofs of the theorems are provided in the appendix.

\section{Notations and preliminaries}
Let $(\mathcal{X}, d_{\mathcal{X}})$, $(\mathcal{Y}, d_{\mathcal{Y}})$, and $(\mathcal{Z}, d_{\mathcal{Z}})$ be metric spaces. Capital letters $X$, $Y$, $Z$ %
are used for random variables taking values in measurable spaces $(\mathcal{X}, \mathcal{B}(X))$, $(\mathcal{Y}, \mathcal{B}(Y))$, $(\mathcal{Z}, \mathcal{B}(\mathcal{Z}))$, where $\mathcal{B}$ denotes the Borel $\sigma$-algebra.
Lower case letters $x$, $y$ and $z$ represent realizations of these random variables.
Super scripts are used for conditional distributions and expectations; e.g. $P_X^y$ is the conditional distribution of $X$ after observation $Y=y$ and $\EX_X^y[X]$ is its expectation. The KL-divergence of $Q_X$ to $P_X$ %
is denoted by $\KLdiv{P_X}{Q_X}$.  Mutual information is defined as $I(X;Y)=\KLdiv{P_{XY}}{P_X\otimes{P_Y}}$ where $P_X$ and $P_Y$ are marginal distributions of $P_{XY}$. The conditional mutual information is defined as $I(X;Y|Z)=\EX_Z[I^Z(X;Y)]$ in which for all $z$, $I^z(X;Y)=\KLdiv{P_{XY}^z}{P_X^z\otimes P_Y^z}$.
All logarithms are in natural base and all information-theoretic quantities are in nats.

In a machine learning problem, an input-output pair is represented by $Z=(X,Y)\in \mathcal{Z}=\mathcal{X}\times\mathcal{Y}$. A distribution $P_Z^w$ from the parameterized family of distributions $\mathcal{M}=\{P_Z^w | w \in \mathcal{W}\}$ is assumed to be describing the data generating process. The value of $w$ is unknown. However, a training set  of $n$ i.i.d. pairs $z^n=\{(x_i,y_i)\}_{i=1}^n$
is observed along with a test sample $x$. The goal is to estimate a value $\hat{y}\in\mathcal{Y}$ which minimizes the expected risk $\EX^{wx}[\ell(\hat{y},Y)]$, where $\ell:\mathcal{Y}\times\mathcal{Y} \to \mathbb{R}$ is the loss function. A learning algorithm is represented as a function $\alg:\mathcal{Z}^n\times \mathcal{X} \to \mathcal{Y}$
which receives $n$ training pairs $z^n$ and a test input $x$ and generates $\hat{y}=\alg(z^n,x)$. The function $\hat{f}(.)=\alg(z^n,.)$ is called the learned hypothesis. If $w$ was known, the best possible function would have been
\begin{equation*}
f^*_w =\argmin_{f\in \mathcal{F}} \EX^{w}_{XY}[\ell(f(X),Y)],
\end{equation*}
where $\mathcal{F}$ is the set of all measurable functions $f:\mathcal{X}\to\mathcal{Y}$. Through out the paper we assume that such optimizations are achieved. Under the assumption that the best function belongs to some $\mathcal{H}\subset\mathcal{F}$, the  optimization can be restricted to $\mathcal{H}$. 
The excess risk of a learning algorithm $\alg$ with respect to a distribution $P_Z^w$ is defined as 
\begin{align*}
e(\alg,w) =\EX^{w}_{Z^nXY}[\ell(\alg(Z^n,&X),Y)]\nonumber\\&-\EX^{w}_{XY}[\ell(f^*_w(X),Y)].
\end{align*}
With some misuse of notations, we also use the notation $e(\hat{f},w)$ to represent the excess risk of the learned hypothesis $\hat{f}$.
The (metric) space of all learning algorithms $\alg:\mathcal{Z}^n\times \mathcal{X} \to \mathcal{Y}$ is denoted as $\cA$. 
A randomized learning strategy is characterized by a distribution $P_\cA$ on $\cA$.
To use the randomized learning strategy, first, a sample $\alg\distas{}P_\cA$ is drawn, and then it is used to generate the (random) hypothesis $\hat{f}(.)=\alg(z^n,.)$.
A randomized learning algorithm can equivalently be described either by $P_\cA$ or the conditional distribution $P_{\hat{f}}^{z^n}$. While the latter is more common in the literature, the former more directly captures the game-theoretic nature of a mixed strategy in minimax game (as will be evident in the next section).

\section{Game of excess risk and dilemma of playing order}

\label{sec:excess_risk_fair_game}
In frequentist learning, usually a minimax bound is desired in which a learning algorithm performs well regardless of the distribution used. This can be described by the following minimax problem
\begin{equation}
\label{eq:freq_minimax_game_deterministic}
\adjustlimits \inf_{\alg \in \cA} \sup_{w\in \cW} e(\alg,w).
\end{equation}
This problem can be viewed as a two-player zero-sum game between algorithm "\emph{designer}" and the "\emph{world}". The designer plays first by selecting an algorithm $\alg$ and the world selects the worst distribution $P_Z^w$ for the particular algorithm which was chosen. 

In game theory, a game is said to have a \emph{value}, if the order of the players can be swapped (in which the value of the game is the (shared) solution of the optimizations). 
Note that this is not always the case. Actually,  in general the second player is in an advantage: she can fine tune her move after observing the move played by the first player.
The standard minimax game of excess risk presented in 
(\ref{eq:freq_minimax_game_deterministic})
is one of the (extreme) cases in which the order of play can not be changed:
if the world plays first, the designer selects the fixed hypothesis $f^*_w$, always achieving the minimum possible excess risk of $0$ (recall that $\mathcal{A}$ is the set of all learning algorithms, and here it is enough to select a learning algorithm $\alg$ which always generates $f^*_W$; i.e. $\alg(z^n,.)=f^*_W(.); \forall z^n \in \mathcal{Z}^n$ ).

But note that the original minimax game in which the world plays after the designer is not completely desirable neither. The reason is that in a usual learning problem the world 
selects a distribution independent of the algorithm used by the designer. But in standard minimax setting presented in optimization
\eqref{eq:freq_minimax_game_deterministic},
it is assumed that
the world will take into account the selected learning algorithm and chose the worst possible distribution for that particular algorithm. This is not a realistic scenario in most applications. That's probably the reason B. Efron states that "there couldn’t be a more pessimistic and defensive theory than minimax" \cite{efron1998ra}.

These arguments demonstrate a dilemma about playing order in the standard minimax formulation
of optimization \eqref{eq:freq_minimax_game_deterministic}. Settling this dispute between the world and designer (about which one to play first) is our main motivation in this section and one of the main contributions of the paper. To do that we try to extend the game 
in a natural way in which the order of the play does not matter.

In the previous century, there has been a variety of results in game theory about sufficient conditions under which the order of the play can be changed in a zero-sum game.
But, it seems that the most appropriate result to start our discussion is the the original work of von Neumann which started this field.
The landmark result of von Neumann states that: when possible actions of players are finite sets, if players are allowed to make random decisions (called mixed strategy in game theory),
the order of the play can be changed \cite{neumann1928theorie,morgenstern1953theory}.

For now, consider that the set of actions are finite.
In the setting of learning theory, allowing mixed strategies means to consider the optimization
\begin{equation}
\label{eq:freq_minimax_game_random}
\adjustlimits\inf_{P_\cA} \sup_{P_W} \EX_{\alg\sim P_{\cA}}\EX_{W\sim P_W}[ e(\alg,W)],
\end{equation}
in which $P_\cA$ and $P_W$ are distributions on $\cA$ and $\mathcal{W}$ respectively. Note that allowing a distribution on algorithms is basically the same as allowing a randomized algorithm $P_{\hat{f}}^{Z^n}$. Also note that allowing the world to select a distribution on $W$ does not change the primal minimax problem (\ref{eq:freq_minimax_game_deterministic}). On one hand, by selecting a delta measure as $P_W$, it can play as in optimization (\ref{eq:freq_minimax_game_deterministic}). On the other hand, it is easy to see that the second player always has a deterministic optimal strategy. 
The benefit of allowing a prior $P_W$ is that
(under some conditions), 
the world can play first (dual optimization) and still achieve the same results. An interesting property of this dual optimization is that when the world selects $P_W$, the designer (as the second player) will face a Bayesian problem for which $P_W$ is known. As such, the dual problem is much easier to study: it's enough to show existence of an algorithm for each specific $P_W$ (this will be used in the proof of Theorem~\ref{theorem:info_bounds_after_minimax}). %
Thus, such a result would also 
describe
a duality between frequentist setting and the Bayesian viewpoint under least favorable prior (please see Section~\ref{sec:bayes-freq-duality} for more discussions about this).

Unfortunately, despite the result of von Neumann on finite sets of actions, 
swapping the order of infimum and supremum in a randomized game is not always possible. Actually, there are quite simple games in which the action space is a closed interval in $\mathbb{R}$ but the second player always wins unless some constraints are put on the problem
(e.g. see \cite{parthasarathy1970games} and Section~26.1 of \cite{rakhlin2014lecturenotes}). In the previous century, much work was done on generalizing the minimax theorems (for survey you can see \cite{simons1995minimax,frenk2004equivalent}). However, it is proved that some forms of (weakly-) compactness and continuity are necessary for the minimax theorem to hold for general continuous sets \cite{simons1989you}. The game of "excess risk" is no exception to this rule. In the next section we will try to find general conditions under which the 
minimax theorem holds for the excess risk.
We will demonstrate the 
by putting some (rather weak) conditions on the action set of the world, it is possible to have a minimax result.

\subsection{Minimax theorems for excess risk}
As mentioned before, since the set of actions in the original game of excess risk (Eq.~\eqref{eq:freq_minimax_game_deterministic}) are not finite, the original von Neumann's minimax theorem can not be used to study it. Actually, not only the actions sets are not finite, but they are quite complicated sets with no natural structure (e.g. they do not poses a simple  finite dimensional linear structure). 
As such, it is necessary to chose an appropriate topology on these sets and at the same time use a general form of minimax Theorem which is capable of handling such abstract structures. 
The next theorem achieves this goal. 

\begin{theorem}
\label{theorem:minimax_general_compact_measures}
Let $\mathcal{W}$ be a compact set of measures on $\mathcal{Z}$ equipped with the Total Variation distance, $\Delta(\mathcal{W})$ be the set of all Borel probability measures on $\mathcal{W}$, and $\Delta(\cA)$ be the set of all probability measures on the space of algorithms. Consider a general bounded loss $l: \mathcal{Y} \times \mathcal{Y} \to [0, B]$. Under these conditions, we have
\begin{align*}
\adjustlimits \sup_{P_W \in \Delta(\cW)} \inf_{P_\cA \in \Delta(\cA)}  &\EX_{\alg,W}  [e(\alg,W)]\\&=
\adjustlimits \inf_{P_\cA \in \Delta(\cA)} \sup_{P_W \in \Delta(\cW)} \EX_{\alg,W} [e(\alg,W)],\nonumber
\end{align*}
in which for a probability measure $W \in \mathcal{W}$ and algorithm $\alg$, the excess risk is defined as
\begin{align}
\label{eq:ER}
e(\alg,W) =\EX_{Z^nZ \sim W^{\otimes n+1}}[&\ell(\alg(Z^n,X),Y)]\nonumber\\&-\EX_{Z \sim W}[\ell(f^*(X),Y)].
\end{align}
\end{theorem}
In this theorem, with some misuse of notation, $ W$ represents the distribution generating $Z$ (i.e. $P_Z^W$).
In order to prove this result, a general minimax theorem is used which can be found in \cite[Theorem (7.1)]{PredLearningGames} (see \ref{theorem:general-minimax}). In general, for the minimax theorem to hold, one needs both compactness and continuity of the loss on one of the action sets. Here, we have chosen to constrain the set of distributions $\mathcal{W}$ (action set of the world). Now to define the topology of this set, it is important to note that compactness is in odds with continuity of the functions: when one strengthens the topology, while it would be easier to have compact sets, it gets harder to have continuous functions. The topology defined by total variation distance provides a good trade off between these qualities. The recent result presented in \cite[Theorem 1]{xu2020continuity} is used to prove the continuity of excess risk with respect to total variation distance.

A minimax theorem for the case when $\Delta(\mathcal{W})$ is compact in the Wasserstein distance is also proved in Appendix \ref{sec:wasserstein}. 

The previous theorem does not require the measures to be from a parametric set of distributions. If one restrict the attention to the cases where the density exists and it is from a parametric family with parameter $W\in \mathcal{W}$, it is possible to achieve the following result.

\begin{theorem}
\label{theorem:minimax_used_with_xu2020}
Let $(\cW, d_\cW)$ be a compact metric space and $(\cA, d_{\cA})$ be a metric space of algorithms. Let $\Delta(\cW)$ and $\Delta(\cA)$ be the set of all Borel probability measures on $(\cW, d_{\cW})$ and $(\cA, d_{\cA})$ respectively. Under the conditions that $\cX$ and $\cY$ are bounded sets, the loss function is bounded, the data generating distribution $P_{Z|W=w}$ is absolutely continuous with respect to the Lebesgue measure for all $w \in \cW$, and its density $f(x, y | w)$ is bounded and continuous in $w$ for all $(x, y) \in \cX \times \cY$, we have
\begin{align}
\label{eq:minimax_duality}
\adjustlimits \sup_{P_W \in \Delta(\cW)} \inf_{P_\cA \in \Delta(\cA)}  &\EX_{\alg W}  [e(\alg,W)]\\&=
\adjustlimits \inf_{P_\cA \in \Delta(\cA)} \sup_{P_W \in \Delta(\cW)} \EX_{\alg,W} [e(\alg,W)]\nonumber.
\end{align}
\end{theorem}

\subsection{When do minimax theorems hold?}

In Theorem \ref{theorem:minimax_general_compact_measures}, we considered a set of probability measure $\cW$ on a metric space $\mathcal{Z}$ and assumed that $\cW$ is compact with respect to the Total Variation distance. In this section, we present examples of problems in which such constraints hold.

Consider the realizable binary classification problem with a hypothesis class $\mathcal{F}$ of functions from $\cX$ to $\cY = \{0, 1\}$.  Let $\mathbb{P}_{P_X, f} = P_X \otimes P_Y^X(f)$ such that $P_Y^{X=x}(f) = \delta(f(x))$ and define $\mathcal{W}$ as
\begin{equation*}
\cW = \{\mathbb{P}_{P_X, f}\, |\, P_X \in \mathcal{P},  f \in \mathcal{F}\}.
\end{equation*}

For this class of problems, we can prove the following theorem.

\begin{theorem}
\label{theorem:compactness_W_VC}
Assume that $\mathcal{P}$ is a compact set of distributions in Total Variation distance, and $\mathcal{F}$ is a compact subset of $L^2(\mathrm{Leb})$. Also assume that every $P_X \in \mathcal{P}$ is absolutely continuous and their density functions $p_X$ are uniformly bounded by a constant $c$. Under these conditions, the set
\begin{equation*}
\cW = \{\mathbb{P}_{P_X, f}\, |\, P_X \in \mathcal{P},  f \in \mathcal{F}\}
\end{equation*}
is compact in Total Variation.
\end{theorem}

Hence, we have shown that in a binary classification problem, under the conditions that the set of distributions 
    $\mathcal{P}$ on $\cX$ is compact in Total Variation, and is absolutely continuous with uniformly bounded densities,
    and the hypothesis class $\mathcal{F}$ is a compact subset of $L^2(\mathrm{Leb})$,
according to Theorem \ref{theorem:compactness_W_VC} and Theorem \ref{theorem:minimax_general_compact_measures}, the minimax equality holds. 

The following Theorem, proves that a large number of parametric families are compact in Total Variation.

\begin{theorem}
\label{theorem:TV_compactness_general}
Let $\cW = \{P_\gamma\; | \; \gamma \in \Gamma\}$ be a parametric family of absolutely continuous probability measures on $\mathbb{R}^n$. Assume that $(\Gamma, d_\Gamma)$ is a compact metric space, the density function $f(x|\gamma)$ is continuous in $\gamma$ for any $x \in \mathbb{R}^n$, and there exists 
an integrable function $g:\mathbb{R}^n \to \mathbb{R}$ such that $|f (x|\gamma_1) - f (x|\gamma_2)| \leq g(x)$ for any $\gamma_1, \gamma_2 \in \Gamma$. Under these conditions, $\mathcal{W}$ is compact in the Total Variation distance.
\end{theorem}

The parametric family of $d$-dimensional normal distributions $\mathcal{P} = \Big\{\mathcal{N}(0, \Sigma) \Big| \Sigma_H \succcurlyeq \Sigma \succcurlyeq \Sigma_L \Big\},$
where $\det(\Sigma_L) > \epsilon$ satisfies the conditions of Theorem \ref{theorem:TV_compactness_general}. The proof can be found in Appendix \ref{sec:gaussian_proof}.

\subsection{Duality between frequentist and Bayesian settings}
\label{sec:bayes-freq-duality}
We call the maxmin problem, the dual representation of the minimax problem. As said earlier, whenever the equality holds, it is said that the game has a value and the value is the solution of either of the optimizations. Note that the Bayesian problem has appeared inside the dual representation.
Using the dual problem, if the prior $P_W$ is fixed, the infimum can be studied using the techniques of \cite{xu2020minimum}. To establish general upper bounds on the value of the game, it is enough to take a supremum over $P_W$ on the obtained upper bounds (see Theorem~\ref{theorem:info_bounds_after_minimax}).

\label{sec:related_literature}
The duality between frequentist minimax and the Bayesian settings is also observed in the classic parameter estimation literature when studying the least favorable priors (see chapter 5 of \cite{lehmann2006theory}). 
The least favorable priors are the priors in the Bayesian setting resulting the worst expected error for the best estimator (which depends on the prior). Changing the order of min/max in that setting is described as the case where the best estimator for the least favorable prior matches the best estimator for the minimax prior (even though the problem is simpler in that setting, it is easy to find problems in which the min/max can not be swapped for non-convex loss functions if the usual deterministic estimators are used).
It should be noted that the classic result of the redundancy-capacity theorem in the universal source coding can be seen as a special (unsupervised) case where the loss function is the log-loss (sometimes called self-information loss in this context) \cite{merhav1998universal}. 

To the best of our knowledge, this is the first time  general minimax results for 
excess risk are presented.

\section{The information-theoretic view on minimum excess risk}
Now that we studied the conditions for which the minimax excess risk is the Bayesian minimum excess risk under the worst prior, we are ready to provide information theoretic bounds on excess risk. Before that, it is important to discuss the importance of studying excess risk from information theoretic standpoint.
\subsection{Generalization gap, excess risk, and No Free Lunch Theorem}
\label{sec:gengap_excess_NFLT}
As was mentioned earlier, most of the previous studies on relation between information theory and learning theory, focused on the generalization gap and not the excess risk.
While there is a direct relation between generalization gap and excess risk in classic ERM algorithms, it is important to note that there is no such relation in general.
To see this, let's start with a general overview of the classic excess risk analysis for ERM algorithms on a hypothesis class $\mathcal{H}$.
This discussion helps to understand the importance of studying excess risk in the information-theoretic setting. 
When analyzing the excess risk of ERM on a hypothesis space, the problem is reduced to finding uniform bounds on the generalization gap \cite{shalev-shwartz_understanding_2014}. The reason is based on a simple decomposition of excess risk into four terms
\begin{align}
	\nonumber
	\ell(\hat{f}(X),Y)-\ell(&f^*_w(X),Y)= \;\,
	\ell(\hat{f}(X),Y)-\hat{\EX}[\ell(\hat{f}(X),Y)]
	\\\nonumber
	&+\hat{\EX}[\ell(\hat{f}(X),Y)]-\hat{\EX}[\ell(h^*(X),Y)]
	\\\nonumber
	&+\hat{\EX}[\ell(h^*(X),Y)]-\ell(h^*(X),Y)
	\\\nonumber
	&+\ell(h^*(X),Y)-\ell(f^*_w(X),Y),
\end{align}
where $\hat{\EX}$ is the expectation with respect to the empirical distribution, $\hat{f}$ is the hypothesis learned by the ERM algorithm  and $h^* = \argmin_{h\in \mathcal{H}} \EX^{w}_{XY}[\ell(h(X),Y)]$.
In order to control the excess risk, it is enough to control the (absolute) value of each of these terms. 
Note that the first and the third terms can both be controlled by a uniform bound on the generalization gap. 
The second term is never positive for an ERM algorithm and thus can be eliminated to achieve an upper bound\footnote{Note that for an ERM algorithm
$\alg$, the learned hypothesis
$\hat{f}(.)=A(z^n,.)$ is a solution for the minimization $\min_{h\in \mathcal{H}} \hat{\EX}[\ell(h(X),Y)]$.}. 
The last term can not be analyzed unless some assumptions are made about the set of distributions $\mathcal{M}$. This is a direct consequence of the No Free Lunch Theorem (NFLT) \cite{wolpert_lack_1996}. Here one can assume that $f^*_w\in\mathcal{H}$ for all $w\in\mathcal{W}$, to eliminate this term.

In order to see what we should expect from an information-theoretic treatment of excess risk, it is illuminative to asses this classic reasoning for the more general case.
First of all, note that the same reasoning does not work 
for non-ERM algorithms which are the main subject of study when information-theoretic bounds are used.
The reason is that the introduction of $h^*$ does not help for non-ERM algorithms to begin with (since the second term can not be eliminated).
Thus, it is clear that a different approach is required to tackle the problem in that setting. More importantly, the arguments remind the reader that we will eventually need some kind of assumptions on the data generating distribution (as dictated by NFLT). As the function $h^*$ is out of the picture for non-ERM algorithms, the nature of such assumptions can be quite different from the one used above. 

The recent work of \cite{xu2020minimum} can be viewed in light of these arguments. The Bayesian setting which is studied in this paper provides a comfortable ground to introduce the assumptions required to control the excess risk. More interestingly, the assumptions required to control the excess risk were turned out to have an information-theoretic nature. %
Next theorem summarize some of the results provided in \cite{xu2020minimum} and \cite{hafez2021rate}.
Note that in these papers the Bayesian learning is studied, for which the distribution $P_W$ is known.

\begin{theorem}%
\label{theorem:xu2020_main}
Assume that $P_W$ is given.

(a) Let $\mathcal{Y}\subset \mathbb{R}^p$ and consider the loss function $\ell(y,y')=\norm{y-y'}^2$. If for all $y\in \mathcal{Y}$ we have $\norm{y}^2\le B$, then there exists a randomized learning algorithm $P_{\hat{f}}^{z^n}$ such that
\begin{align*}
\EX_{W,\hat{f}}[e(\hat{f},W)]\le 2B I(W;Y|X,Z^n).
\end{align*}

(b) Consider a bounded loss $\ell:\mathcal{Y}\times\mathcal{Y} \to [0,B]$ and assume that $\EX^{w}_{XY}[\ell(f^*_w(X),Y)]=0$ for all $W\in \cW$. There exists a randomized learning algorithm $P_{\hat{f}}^{z^n}$ such that
\begin{align*}
\EX_{W,\hat{f}}[e(\hat{f},W)]\le& 3B I(W;Y|X,Z^n)
\end{align*}

(c) Consider the general bounded loss $\ell:\mathcal{Y}\times\mathcal{Y} \to [0,B]$.
There exists a randomized learning algorithm $P_{\hat{f}}^{z^n}$ such that
\begin{align*}
\EX_{W,\hat{f}}[e(\hat{f},W)]\le& B \sqrt{\frac{1}{2}I(W;Y|X,Z^n)}.
\end{align*}
Note that the term $I(W;Y|X,Z^n)$ which appears in all the bounds can be further bounded by both 
 $\frac{1}{n} I(W;Y^n|X^n)$
and
$ \frac{1}{n} I(W;Z^n)$. Furthermore, under some regularity conditions, it can be shown that $I(W;Y|X,Z^n)$ has the rate $O(1/n)$ as $n \to \infty$ (see Lemma~\ref{lemma:fisher}).
\end{theorem}

\subsection{Upper bounds on minimax excess risk}
While Theorem~\ref{theorem:xu2020_main} provides an interetsing information theoretic view on the minimum excess risk, 
the Bayesian requirement that $P_W$ is known is too restrictive from the frequentist standpoint. Using the game theoretic terminology discussed in Section \ref{sec:excess_risk_fair_game}, this means that the designer is the second player and can design the algorithm for the specific $P_W$ which was chosen by the world. But in the minimax setting, the designer should design a single (possibly random) algorithm which works for all the distributions. 
This is where the proposed minimax theorems and the duality between the frequentist and Bayesian settings come to the rescue by proving the \emph{existence} of such universal algorithms.

In next theorem, we can combine Theorem~\ref{theorem:xu2020_main} and Theorem~\ref{theorem:minimax_used_with_xu2020} to provide information-theoretic bounds on minimax excess risk. Assume that the capacity of the channel $P_{Z^n}^w$ is bounded by $\kappa_n$; i.e.  
\begin{equation}
\label{eq:capacity_definition}
\mathcal{C}(P_{Z^n}^w)\triangleq \max_{P_W} I(W;Z^n) \le \kappa_n.
\end{equation}

\begin{theorem}
\label{theorem:info_bounds_after_minimax}
Under the assumptions of Theorem~\ref{theorem:minimax_used_with_xu2020}, we have the following results.

(a) Consider $\mathcal{Y}\in \mathbb{R}^p$ and loss function $\ell(y,y')=\norm{y-y'}^2$. If for all $y \in \mathcal{Y}$ we have $\norm{y}^2\le B$, then 
\begin{align*}
\adjustlimits \inf_{P_\cA \in \Delta(\cA)} \sup_{P_W \in \Delta(\cW)} &\EX_{\alg,W} [e(\alg,W)]\\
\le& 2B  \sup_{P_W \in \Delta(\cW)} I(W;Y|X,Z^n)
\\
\le& 
\frac{2B\kappa_n}{n}.
\end{align*}

(b) Consider a bounded loss $\ell:\mathcal{Y}\times\mathcal{Y} \to [0,B]$ and assume that $\EX^{w}_{XY}[\ell(f^*_w(X),Y)]=0$ for all $W\in \cW$. We have
\begin{align*}
\adjustlimits \inf_{P_\cA \in \Delta(\cA)} \sup_{P_W \in \Delta(\cW)} &\EX_{\alg,W} [e(\alg,W)]\\
&\le 3B  \sup_{P_W \in \Delta(\cW)} I(W;Y|X,Z^n)
\\
&\le
\frac{3B\kappa_n}{n}.
\end{align*}

(c) Consider the general bounded loss $\ell:\mathcal{Y}\times\mathcal{Y} \to [0,B]$. We have
\begin{align*}
\adjustlimits \inf_{P_\cA \in \Delta(\cA)} \sup_{P_W \in \Delta(\cW)} &\EX_{\alg,W} [e(\alg,W)]\\
&\le B   \sqrt{\frac{1}{2} \sup_{P_W \in \Delta(\cW)}I(W;Y|X,Z^n)} 
\\
&\le  B   \sqrt{\frac{\kappa_n}{2n}}.
\end{align*}
\end{theorem}
This theorem shows that there exists a single (randomized) algorithm which uniformly performs well on all the 
distributions in $P_W$
given that $I(W;Y|X,Z^n)$ or the channel capacity $\mathcal{C}(P_{Z^n}^W)$ is bounded. Note that the channel capacity only depends on the conditional distribution $P_{Z^n}^W$ and thus it provides a complexity measure for the space of all the distributions in the considered family of distributions. In the next sections, we will demonstrate that in some interesting cases $\kappa_n$ grows with the rate $O(\log n)$.

\subsection{Asymptotic analysis}
The terms $I(W;Y|X,Z^n)$ and $I(W;Z^n)$ can be analyzed for large $n$ using classic results of \cite{clarke1990bayes, clarke1994jeffrey}. This is summarized in the following Lemma. 

\begin{lemma}
	\label{lemma:fisher}
	Let $\mathcal{W}\subset \mathbb{R}^d$ be a compact set and the model $P_{Z}^w$ be \emph{smooth} in $w$ with the Fisher information matrix $J_{Z}^w(w)$. Then, as $n \to \infty$, we have
	\begin{align}
	\label{eq:fisher}
	I(W; Z^n) =\frac{d}{2} \log\big(\frac{n}{2\pi e}\big) &+ h(W) \nonumber\\ &+ \EX \big[\log| J_{Z}^W (W)|\big] + o(1),
	\end{align}
	and 
	\begin{equation}
	\label{eq:fisher_tighter}
	I(W;Y|X, Z^n) = O\Big(\frac{d}{2n}\Big).
	\end{equation}
\end{lemma}
The exact smoothness conditions can be found in the appendix.

This result was used by \cite{xu2020minimum} to derive asymptotic rates for the minimum excess risk in the Bayesian setting. It might seem appealing to do the same here; i.e. substitute r.h.s. of \eqref{eq:fisher_tighter} inside 
the inequalities of Theorem~\ref{theorem:info_bounds_after_minimax}
in order to get rates of order $O(1/n)$ and $O(\sqrt{1/n})$. Unfortunately, that's not possible here as the convergence rate of the limit in Lemma~\ref{lemma:fisher} is not uniform on all distributions. To see the problem, note that for each value of $n$, the world is allowed to select a different distribution. This shows another dissatisfactory element in minimax analysis. We usually like to know the behaviour of algorithms for a fixed distribution as $n$ grows.
This kind of problem is well known in approximation theory and has been extensively analyzed (e.g. see \cite[Sect. 2]{temlyakov2003nonlinear} and \cite{CaponnettoOptimalratesregularized2007}). A resolution to this is to consider the individual upper rate of convergence which is defined as a sequence $(b_n)_{n\ge 1}$ of positive numbers such that 
\begin{equation*}
\inf_{\{P_{\cA}^n\}_{n\in\mathbb{N}}} \sup_{P_W\in \Delta(\cW)} \Psi(\{P_{\cA}^n\}_{n\in\mathbb{N}}; P_W) < \infty
\end{equation*}
with 
\begin{align*}
    \Psi(\{P_{\cA}^n\}_{n\in\mathbb{N}}&; P_W) = \\
    &\limsup_{n\to +\infty} \frac{\EX[\ell(\alg(Z^n,X),Y)]-\EX[\ell(f^*(X),Y)]}{b_n}.
\end{align*}
Here, the infimum is over the set of randomized learning algorithms $\{P_{\cA}^n\}_{n\in\mathbb{N}}$. On the other hand, if 
\begin{equation*}
\inf_{\{P_{\cA}^n\}_{n\in\mathbb{N}}} \sup_{P_W\in \Delta(\cW)} \Psi(\{P_{\cA}^n\}_{n\in\mathbb{N}}; P_W) > 0,
\end{equation*}
then, $b_n$ is called an individual lower rate.
The next theorem states the individual upper rate for learning the class of distributions for which Lemma~\ref{lemma:fisher} holds.

\begin{theorem}
\label{theorem:individual_upper_rate}
Under conditions of Theorem~\ref{theorem:minimax_used_with_xu2020} and Lemma~\ref{lemma:fisher}:
the sequence $\frac{1}{n}$ is an individual upper rate for bounded quadratic loss and
the sequence $\frac{1}{\sqrt{n}}$ is an individual upper rate for a general bounded loss.
\end{theorem}

Using the Rate-Distortion framework recently introduced in \cite{hafez2021rate}, individual lower rates for excess risk can also be derived.
Assume that $\mathcal{Y} \subset \mathbb{R}^p$ is  compact and $\ell(y, y') = ||y-y'||^2$. Based on Theorem 10 of \cite{hafez2021rate}, we have
\begin{align*}
 \EX_{\alg W}[ e(&\alg,W)]\\ &\geq \frac{d}{n\pi} \exp\Big(-\frac{2}{d}\big(\EX_{W} \log |J_Z^W(W)| + o(1)\big)\Big),
\end{align*}
for any $P_W$ and $P_Z^W$ satisfying the conditions of Theorem 10 of \cite{hafez2021rate}, and any $\alg$. Under the condition that $\sup_{P_W\in \Delta(\cW)}\EX_{W \sim P_W} \log |J_Z^W(W)| \leq C$, we have that $b_n = 1/n$ is an individual lower rate. Note that the individual lower rates do not need the duality to hold.

\subsection{Information theoretic view on VC-dimension}
\label{sec:VC}
While learning properties of hypothesis spaces with finite VC-dimension are well understood using classical uniform bounds, they are still studied in many information-theoretic papers. There are two benefits of doing that. First of all, this gives more insight into the information-theoretic nature of classic learning theory. Secondly, it is important to see whether the theory works for this important set of problems. 

\subsubsection{A history on information-theoretic analysis of VC classes}
An insightful example is the development of the information-theoretic generalization bound. After the first wave of works on information-theoretic generalization bounds in  \cite{RussoHowmuchdoes2015,XuInformationtheoreticanalysisgeneralization2017}, describing the generalization using the information-theoretic language has gained much attention. After all, it seemed a natural idea to be able to relate the amount of overfitting, to the amount of information the algorithm has saved inside the learned hypothesis.
One of the main problems was the discovery of \cite{BassilyLearnersthatUse2018} showing that the generalization bounds obtained from the information-theoretic viewpoint could be much worse than the ones found by VC theory. More precisely, they considered the simple problem of learning threshold functions when $\mathcal{X}=\{1,\dots k\}$ for a $k\in \mathbb{N}$. They showed that for any randomized algorithm which is ERM, there exists a distribution for which the upper bound on generalization is $\Omega(\frac{\log \log |\mathcal{X}|}{n^2})$. This was not a satisfactory result as the cardinality of $\mathcal{X}$ is usually considered to be unbounded. On such sets, while the VC theory provides fast rates in minimax setting, the information-theoretic bounds were vacuous. 

In the works published in the subsequent years on the subject, some explicitly tried to tackle this problem and provide an information-theoretic view that can also explain VC theory.  
In \cite{AsadiChainingMutualInformation2018a}, the chaining method was combined with the information-theoretic techniques bringing together the capabilities of both techniques to prove generalization bounds. In \cite{steinke_reasoning_2020} a symmetrization  techniques was adapted to improve the bounds. In this setting a super-sample of size $2n$ was considered as given in which the training set is selected from randomly. Then the conditional mutual information of the learned hypothesis and training set appeared in the bounds. Using this technique the authors provided $O(\dvc \log n /n)$ bounds which are loose by the $\log n$ factor. They also made some conjectures about the existence of algorithms for which the $\log n$ factor will not appear \cite{steinke2020open}. Their conjectures cover both agnostic and realizable settings. The conjectures about the realizable setting were later refuted in \cite{haghifam2020information} (and replaced by other conjectures).  In \cite{hafez2020conditioning} the chaining method was combined with the conditional mutual information to remove the $\log n$ factor in the agnostic setting.

Considering the game theoretic nature of arguments, perhaps the most relevant to the current work is \cite{nachum2019average} (even though in that paper excess risk is not studied). 
They take a game-theoretic view on the problem in which the gain of the world is directly defined to be the mutual information between the learned model and the dataset (the term which appears in the generalization bounds). In their setting, the world is directly trying to select a distribution for which the information-theoretic bound on generalization gap is increased. As it was found that the world is too strong to be beaten, it was suggested that its power should be reduced. One of the suggestions was that if the prior probability $P_h$ on the class of functions $\mathcal{H}$ is selected beforehand, the minimax order can be swapped. A setting in which the learning algorithm has to know $P_h$ but knowing $P_X$ is not required. In 
this section
the excess risk of learning of hypothesis classed with finite VC-dimension is studied. It is demonstrated that in our case which the attention is on the excess risk, under quite general conditions the order of play can be swapped and information-theoretic bounds provide a concise and correct evaluation of the hardness of the problem. 
In contrast to the results of 
\cite{nachum2019average}, here the algorithm does not need to know the distribution $P_h$ neither.  It's also worth noting that the reason the bounds provided in current work (in particular Theorem \ref{theorem:info_bounds_after_minimax}) do not suffer from the problem of unbounded mutual information, is the fact that in $I(W;Y|X,Z^n)$, all the continuous variables in training data are among the conditioned variables. More precisely we have $I(W;Y|X,Z^n)\le H(Y|X,Z^n)\le \log |\mathcal{Y}|$. Note that this is in nature similar to how \cite{steinke_reasoning_2020} and \cite{hafez2020conditioning} used the conditioning technique on continuous random variables belonging to $\mathcal{X}$ to solve the problem of unbounded mutual information.

\subsubsection{Excess risk for VC classes}
Consider a class of functions $\mathcal{H}$ with a finite VC-dimension $\dvc$. To use Theorem~\ref{theorem:minimax_general_compact_measures}, consider the class of distributions $\mathcal{W}$ with the property that for each $W\in \mathcal{W}$, there exists $h\in \mathcal{H}$ such that $Y=h(X)$ almost surely. Suppose that $\mathcal{W}$ is compact with respect to the Total Variation metric. Using Theorem~\ref{theorem:minimax_general_compact_measures}, we can replace the order of minimax, thus, it is enough to analyze the problem in the Bayesian setting.
Using Theorem~\ref{theorem:xu2020_main} we know that there exists randomized algorithm $P_{\hat{f}}^{z^n}$ such that 
\begin{align*}
\EX[e\big(\hat{f},w\big)] \le& \,
3 I(W;Y|X,Z^n).
\end{align*}
Now, we have 
\begin{align}
\label{eq:vc_bound_proof}
I(W;Y| X, Z^n) &\le \frac{1}{n}I(W;Y^n|X^n) \le \frac{1}{n}H(Y^{n+1} | X^{n+1})\nonumber\\ &\le  \frac{1}{n}\log \big(e n^{\dvc}\big) = O\Big( \frac{\dvc \log n}{n} \Big). 
\end{align}
The third inequality is derived by first noting that entropy is upper bounded by the logarithm of the maximum number of values $Y^n$ can take when $X^n=x^n$ is given and then using the Sauer-Shelah lemma. Thus, the inequalities (\ref{eq:vc_bound_proof}) provide a simple, yet intuitive proof of the bound of $O(\dvc \log n /n)$ on the excess risk in VC theory. 

Now, the natural question to ask is whether the $\log n$ factor can be removed to meet the tight bounds from VC theory. We conjecture that $I(W;Y^n|X^n) = O(\dvc/n)$ and the presented information-theoretic methodology is enough to show the tight upper bound of $O(1/n)$. The rational behind this conjecture is that in \eqref{eq:vc_bound_proof}, the upper bound of $\frac{1}{n}I(W;Y^n|X^n)$ was used which is not tight. Actually, this is similar to the result presented in Lemma~\ref{lemma:fisher}. Comparing equations \eqref{eq:fisher} and \eqref{eq:fisher_tighter}, we can observe that in the setting of that lemma, the  looseness of such bound was exactly $\log n$. Similarly, we can write
\begin{align*}
I(W;Y | X, Z^n) = I(W;Y^{n+1} | X^{n+1}) - I(W;Y^n | X^{n+1}),
\end{align*}
in which $X^{n+1}=X^n\cup \{X\}$, $Y^{n+1}=Y^n\cup \{Y\}$, and $Z^{n+1}=Z^n\cup \{(X,Y)\}$. 
Now if one could show that whenever the first term is $cD\log(n+1)$ for a constant $c\in \mathbb{R}$, the second term would be $cD\log n$, the result will be achieved by
noting that $I(W;Y | X, Z^n)\le \log ((n+1)/n) \le O(1/n)$. 
One challenge to achieve this, is the necessity to analyze both of these terms together.
Another difficulty is that the constant $c$ should be the same for both of the terms. 
It's also worth mentioning that the studied minimax setting in this paper allows improper learners. If one restricts the algorithms to proper learners, counterexamples similar to \cite{haghifam2020information} could be provided to refute the conjecture.

\section{Designing algorithms for a restricted class of priors}
The results presented in the previous sections were all about the existence of good learning algorithms. But those results are not constructive, meaning that they could not readily be used to find such learning algorithms. In this section, we discuss a method to do so which results in algorithms that asymptotically meet the previously obtained bounds of $O(\sqrt{1/n})$ (given that the class of priors has a finite radius).

The posterior sampling algorithm $\alg_{\text{PS}(P_0)}$, is a learning algorithm that uses the prior $P_0$  to calculate the posterior distribution $P_W^{xz^n}$ given training samples $Z^n = z^n$ and $X=x$. Then it samples $W'$ from the calculated posterior, and then sets $\alg_{\text{PS}(P_0)}(z^n, x) = f^*_{W'}(x)$.

Now assume that we fix a particular set of priors $\mathcal{M}_W$ on $W$. The following theorem, is an extension of Theorem~\ref{theorem:xu2020_main}.

\begin{theorem}
\label{thm:posterior_Samlping}
Consider a set of priors $\mathcal{M}_W$ on $W$. Assume that this set has a finite radius $$r=\adjustlimits\inf_{P'_W\in\mathcal{M}_W} \sup_{P_W\in\mathcal{M}_W} \KLdiv{P_W}{P'_W},$$
and  that the infimum is achieved at $P_0$.
Consider the posterior sampling algorithm $\alg_{\text{PS}(P_0)}$. If for all $z^n$ and $x$, the cumulant generating function of the random variable $\ell(\alg_{\text{PS}(P_0)}(z^n,x),Y)$ exists for $\lambda \in [0, b)$ and is upper bounded by $\varphi(\lambda)$ under all distributions $P_W\in \mathcal{M}_W$, we have
\begin{align*}
\sup_{P_W\in\mathcal{M}_W} &\EX_{W\sim P_W} [e(\alg_{\text{PS}(P_0)},W)] \\
&\le  \varphi^{*-1} \Big(\frac{r}{n} +\sup_{P_W\in\mathcal{M}_W} I(W;Y|XZ^n)\Big), 
\end{align*}
where $\varphi^{*-1}$ is the generalized-inverse of $\varphi^*(\gamma) = \sup_{0\leq \lambda < b} \big(\lambda \gamma - \varphi(\lambda)\big)$.
\end{theorem}

This theorem states that when the distribution $P_W$ is not known but is chosen from a set with finite radius, one can use the posterior sampling algorithm with the distribution $P_0$ to asymptotically achieve similar excess risks.

\begin{remark}
Assume that $\mathcal{M}_W = \{P_{W}^{\theta}\; |\; \theta \in \Theta\}$ is a parametric family of distributions in which $\Theta$ is the parameter space. Consider an arbitrary distribution $P_\theta$ on $\Theta$. It is well known that 
\begin{equation*}
    I(\theta;W) \leq {\adjustlimits \inf_{\theta' \in \Theta} \sup_{\theta \in \Theta} \KLdiv{P_W^{\theta}}{P_W^{\theta'}}},
\end{equation*}
for any $P_\theta$ \cite{PolyanskiyLectureNote}. Hence, having a finite radius implies a finite capacity from $\theta$ to $W$.
\end{remark}

\subsection{Regularized least squares}

In this section, we will use the machinery developed so far, to study the excess risk of regularized least squares. To do so, we will connect the excess risk of least squares and posterior sampling.

Let $\mathcal{M}_W = \{\mathcal{N}_{d\times d}(\mu, \sigma_w^2 I)\;|\; ||\mu||_2 \leq c \}$ be a restricted class of $d$-dimensional normal priors on $W$. This family of distributions has a finite radius $r = \frac{c}{2\sigma_w^2}$ and $P_0 = \mathcal{N}(0, \sigma_w^2 I)$.  Let $P_W = \mathcal{N}(\mu, \sigma_w^2 I) \in \mathcal{M}_W$ and consider the following data generating process. Let $X\sim P_X$ and $Y=W^\top\phi(X)+E$, in which $\phi:\mathcal{X}\to\mathbb{R}^d$ is a known function  and $E\distas{}\mathcal{N}(0,\sigma_e^2)$ is an independent zero-mean Gausssian noise. 

\begin{theorem}
\label{RLS:MI}
Given $n$ i.i.d. training pairs $z^n = \{(x_i, y_i)\}_{i = 1}^{n}$, let ${\boldsymbol{\Phi}} \in \mathbb{R}^{n\times d}$ and $\mathbf{Y} \in \mathbb{R}^n$ be the matrix of samples $\{\phi(x_i)\}_{i = 1}^{n}$ and $\{y_i\}_{i = 1}^{n}$ respectively. Also denote the eigenvalues of $\frac{1}{n}(\boldsymbol{\Phi}^\top\boldsymbol{\Phi})$ and $\EX_{X\sim P_X}[\phi(X)\phi(X)^\top]$ with $(\hat\sigma_i)_{i = 1}^{d}$ and $(\sigma_i)_{i = 1}^{d}$ respectively. We have 
\begin{align}
\label{eq:RLS_cond_dist_is_normal_man}
    P_W^{z^n} =  \mathcal{N}(\mu_W^{z^n}, \Sigma_W^{z^n}),
\end{align}
with
\begin{align*}
&\mu_W^{z^n} = \bigg(    \boldsymbol{\Phi}^\top\boldsymbol{\Phi} + \frac{\sigma_e^2}{\sigma_w^2}I \bigg)^{-1} \Big(\boldsymbol{\Phi}^\top\mathbf{Y} +\frac{\sigma_e^2}{\sigma_w^2} \mu\Big),\\
&\Sigma_W^{z^n} = \bigg( \frac{\boldsymbol{\Phi}^\top\boldsymbol{\Phi}}{\sigma_e^2} + \frac{I}{\sigma_w^2} \bigg)^{-1}.
\end{align*}
Defining $\lambda = \frac{\sigma_e}{\sigma_w}$, it also holds that
\begin{align*}
I(W;Z^n)=\frac{1}{2} &\EX_{Z^n}\Bigg[\log\frac{\sigma^{2d}_w}{\prod_{i=1}^d \frac{\sigma_e^2}{n\hat{\sigma}_i+\lambda}} -d \\ &+ \sum_{i=1}^d \frac{\lambda}{n\hat{\sigma}_i+\lambda} + \frac{1}{\sigma_w^2}\norm{\mu_W^{Z^n} - \mu}^2 \Bigg].
\end{align*}
\end{theorem}

Now, we will show that there is a direct relationship between the excess risk of Posterior sampling in this problem and the excess risk of regularized least squares. Note that in \eqref{eq:RLS_cond_dist_is_normal_man}, we have seen that $\mu_W^{z^n}$ is the regularized least squares estimator. Let $e^*= \EX[\ell(f_W^*(X), Y)]$. We have
 \begin{align*}
     \EX[e(A_{\text{PS}}, W)] &= \EX_{Z^nW_{\text{PS}} \phi(X) Y}\big(W_{\text{PS}}^\top \phi(X) - Y\big)^2 - e^*\\
     &= \EX_{Z^nXY}\EX^{Z^n}_{W_{\text{PS}}}\big(W_{\text{PS}}^\top   \phi(X) - Y\big)^2 -  e^* \\
     &\geq \EX_{Z^n X Y}\big(\EX^{Z^n}_{W_{\text{PS}}}[W_{\text{PS}}]^\top \phi(X) - Y\big)^2 - e^*\\
     &= \EX_{Z^nXY}\big(\mu_W^{Z^n\top} \phi(X) - Y\big)^2 - e^*\\
     &= \EX[e(A_{\text{RLS}}, W)].
 \end{align*}
The second equality follows from the fact that $W_{\text{PS}} \independent{} W | Z^n$, the convexity follows from the convexity of the loss function, and  the third equality follows from $\EX^{Z^n}_{W_{\text{PS}}}[W_{\text{PS}}] = \EX^{Z^n}_{W}[W]$.

Thus, an upper bound on the excess risk of the posterior sampling algorithm is an upper bound on the excess risk of regularized least squares.

Theorem~\ref{thm:posterior_Samlping} can be used to upper bound the excess risk of the posterior sampling algorithm. To use this theorem, we need to find a function $\varphi(\lambda)$, such that the cumulant generating function of $\ell(A_{\text{PS}}(z^n, x), Y)$ exists for $\lambda \in [0, b)$ and is upper bounded by $\varphi(\lambda)$ under all $P_W \in \mathcal{M}_\cW$. This bound follows from the fact that for squared loss, $\ell(A_{\text{PS}}(z^n, x), Y)$ is a  sub-exponential random variable. The exact reasoning and the form of $\varphi(\cdot)$ can be found in Section \ref{sec:theorem-10} of the appendix.

\begin{theorem}
\label{thm:posterior_sampling_RLS}
For the problem described above, the posterior sampling algorithm $\alg_{\text{PS}}(P_0)$ satisfies
\begin{align*}
\label{eq:RLS_upper_bound}
\sup_{P_W\in\mathcal{M}_W} &\EX_{W\sim P_W} [e(\alg_{\text{PS}(P_0)},W)] \\
&\le \sup_{P_W\in\mathcal{M}_W} \varphi^{*-1} \Big(\frac{1}{n}\Big(  I(W;Z^n) + \frac{c}{2\sigma_w^2}\Big)\Big)\\
&\leq O\Big(\sqrt{\frac{d\log(n)}{n}}\Big)
.
\end{align*}
\end{theorem}

Based on Theorem \ref{RLS:MI}, for any $P_W$ (i.e., any choice of $\mu$) as $n\to \infty$, we have
$I(W;Z^n) = O(d\log(n))$. Hence,
\begin{align*}
    \sup_{P_W\in\mathcal{M}_W} \EX_{W\sim P_W} &[e(\alg_{\text{RLS}},W)]\\
    &\leq \sup_{P_W\in\mathcal{M}_W} \EX_{W\sim P_W} [e(\alg_{\text{PS}},W)] \\
    &\le \sup_{P_W\in\mathcal{M}_W} \varphi^{*-1} \Big(\frac{1}{n}\Big( I(W;Z^n) + \frac{c}{2\sigma_w^2}\Big)\Big)\\
    &\leq O\Big(\sqrt{\frac{d\log n}{n}}\Big),
\end{align*}
where the last inequality follows from Theorem \ref{RLS:MI} and Theorem \ref{thm:posterior_sampling_RLS}. It should be noted that while the dimension $d$ has appeared in the bounds, it is actually possible to do much better if the eigenvalues of the covariance matrix are small; i.e., the \textit{effective dimension} is small. To see this, note that 
we can rewrite $I(W;Z^n)$, which was presented in Theorem~\ref{RLS:MI}, as
\begin{align*}
I(W;Z^n)&=\frac{1}{2} \EX_{Z^n}\Bigg[-\sum_{i=1}^d\log \frac{\lambda}{n\hat{\sigma}_i+\lambda} + \sum_{i=1}^d \frac{\lambda}{n\hat{\sigma}_i+\lambda}\\
&~\hspace{2cm}~ -d + \frac{1}{\sigma_w^2}\norm{\mu_W^{Z^n} - \mu}^2 \Bigg]
\\
&\le
\EX_{Z^n}\left[-\sum_{i=1}^d\log \frac{\lambda}{n\hat{\sigma}_i+\lambda} +  \frac{1}{\sigma_w^2}\norm{\mu_W^{Z^n} - \mu}^2 \right],
\end{align*}
in which $\lambda=\sigma^2_e/\sigma^2_w$. The inequality is derived by noting that $\frac{\lambda}{n\hat{\sigma}_i+\lambda}\le 1$. Now, note that if a value of $\hat{\sigma}_i$ is small compared to $\lambda$, its effect is negligible in the mutual information. This makes it possible to have valid bounds even for very large $d$, when one can guarantee that the eigenvalues are suitably small for most of the dimensions. 

In the classical analysis of kernel regularized least squares, where the feature space could be infinite dimensional, a similar notion of effective dimension is defined based on eigenvalues of the operator in the Hilbert space. In that setting, it is shown that one of the requirements for the kernel regularized least squares to work well, is to have fast decays on the the eigenvalues \cite{CaponnettoOptimalratesregularized2007}. To the best of our knowledge, the results provided here, are the first to  provides an information-theoretic insight on effective dimension and why it appears in the bounds on excess risk. In comparison, the most relevant results in the literature are the ones provided in Section 3 of \cite{xu2020continuity}, in which the non-regualarized least squares is analyzed. In that setting, the bounds will be vacuous when the input dimension  is large (e.g.  Eq.~(102) of  \cite{xu2020continuity} when $p$ is large).

It should be noted that, while the provided results in this section work for \emph{any finite} dimension, the mathematical machinery utilized here is not strong enough to directly handle the infinite-dimensional case.
Extending the current information theoretic framework to the general case where the functions are from a (possibly) infinite dimensional Hilbert space seems an interesting direction for future studies.

\section{A discussion about generality of theorems}
In this paper we presented a variety of Theorems each with different degrees of generality. 
In this section a comparison between the assumptions needed for each setting is provided. To make the comparison complete, we also include some of the results of \cite{xu2020minimum} which are in frequentist setting (i.e. Section 7.2 and related parts in Section 5 of that paper). The following equation roughly shows how general the assumptions are compared to each other. It is sorted from the most general to the most strong set of assumptions:
\begin{align*}
\text{Theorem~\ref{theorem:minimax_general_compact_measures}} &\underset{(1)}{>} 
\text{Theorem~\ref{theorem:minimax_used_with_xu2020}} \underset{(2)}{>}
\text{Theorem~\ref{thm:posterior_Samlping}}\\ &\underset{(3)}{>}
\text{Related parts of \cite{xu2020minimum}}.   
\end{align*}

The significant factors considered in each comparison are as follows: 

\textbf{(1)} Theorem \ref{theorem:minimax_general_compact_measures} does not require a parametric family of distributions and also does not require the density functions to be uniformly bounded. But Theorem \ref{theorem:minimax_used_with_xu2020} has all these extra assumptions.

\textbf{(2)} Theorem \ref{theorem:minimax_used_with_xu2020} does not put any restriction on the set of priors the world can choose from.
In contrast, Theorem~\ref{thm:posterior_Samlping} just works for a restricted set of priors with finite diameter and the diameter appears in the bound (the price we had to pay to go beyond pure existence results and provide real algorithms).

\textbf{(3)} Among all the previous restrictions, Section 7.2 of \cite{xu2020minimum} also requires the set of priors to have a finite size and the logarithm of the size will appear in the bound%
\footnote{To see this, note that Corollary 8 of \cite{xu2020minimum} which is used in section 7.2 of that paper to intuitively discuss the relation with the minimax setting, is based on the term 
	$H(M|Z)$
	 (and $R_{01}(M|Z^n)$
	for squared loss). 
		It should be noted that when the index set $\mathbb{M}$ is not finite, for any fixed value of $n$, these terms could be arbitrarily large if the prior on $M$ is spread among many values. The intuitive reason is that in order to exactly recover a continuous parameter, one usually requires an infinite number of samples. To see this more precisely for conditional entropy, use the equality $H(M|Z)=H(M)-I(M;Z)$. Now note that the worst case prior which should be considered in minimax, could potentially set $H(M)=\log |\mathcal{M}|$, which shows that the resulting bound for minimax setting will have an implicit dependence on $\log |\mathcal{M}|$.}.
The used definition of diameter (Eq. 146 of \cite{xu2020minimum}) is also more restrictive as it takes maximum on all priors  all distributions  and all inputs .

\section{Conclusion and future work}
In this paper, information-theoretic bounds on the minimax excess risk were studied. First, it was demonstrated that the recent information-theoretic results presented by \cite{xu2020minimum} for Bayesian learning provide a suitable way to encode the assumptions on the problem and to study the excess risk. However, as a Bayesian method, they assume a distribution on the parameters. To bridge the gap between these methods and the frequentist minimax setting, the problem was decoded as a game in which the frequentist and Bayesian settings appeared as primal and dual optimizations. It was shown that under what conditions, this game has a value and the duality gap is zero. Under these conditions, it was adequate to study the dual problem in which the machinery developed by \cite{xu2020minimum} provides a method to derive bounds on the excess risk. It was demonstrated that these techniques provide a convenient set of tools to study minimum excess risk on problems such as classification on hypotheses classes with finite VC-dimension. It was also demonstrated that this new technique does not suffer from the problems which the information-theoretic bounds on generalization were facing. Finally, the posterior sampling method was studied and it was shown that it can approach the optimal algorithm given that the class of priors is suitably controlled.  The regularized least squares problem was studied with this technique.

There are various problems worthy of further analysis. Firstly, the minimax theorems proved in this paper provide a general structure for problems that we know can be solved. Studying various classes of problems in which an information-theoretic approach is plausible is interesting. Secondly, most of the bounds provided are upper bounds and the lower bounds are less understood at this point. Further analysis of these concepts could shed more light on the information theoretic nature of learning. 

\appendix

To prove the minimax results in Theorem \ref{theorem:minimax_general_compact_measures} and Theorem \ref{theorem:minimax_used_with_xu2020}, we first need to state a general minimax theorem from \cite[Theorem 7.1]{PredLearningGames}.

\begin{theorem}
\label{theorem:general-minimax}
Let $f(p, q)$ denote a bounded real-valued function defined on $\mathcal{P}\times \mathcal{Q}$, where $\mathcal{P}$ and $\mathcal{Q}$ are convex sets and $\mathcal{P}$ is compact. Suppose that $f(., q)$ is convex and continuous for each fixed $q \in \mathcal{Q}$ and $f(p, .)$ is concave for each finxed $p\in \mathcal{P}$. Then 
\begin{equation*}
    \adjustlimits \inf_{p\in \mathcal{P}} \sup_{q\in\mathcal{Q}} f(p, q) = \adjustlimits\sup_{q\in \mathcal{Q}} \inf_{p\in\mathcal{P}} f(p, q).
\end{equation*}
\end{theorem}

\subsection{Proof of Theorem \ref{theorem:minimax_general_compact_measures}}
To prove Theorem \ref{theorem:minimax_general_compact_measures}, we will need to prove the two following lemmas.

\begin{lemma}
\label{lemma:TV-continuity}
Let $f: \mathcal{H} \to [0, B]$ be a bounded function and $P$ and $Q$ be two probability measures on $\mathcal{H}$. We have
\begin{equation}
    \label{eq:TV_upper_bound}
    |\EX_{X\sim P}f(X) - \EX_{X\sim Q} f(X)| \leq B d_{\mathrm{TV}}(P, Q),
\end{equation}
where $d_{\mathrm{TV}}$ is the Total Variation distance.
\end{lemma}
\begin{proof}
The upper bound \eqref{eq:TV_upper_bound} can be show as follows:
\begin{align*}
    \EX_{X\sim P}f(X) &- \EX_{X\sim Q} f(X) \\
    &= \int_\mathcal{H} f(x) P(dx) - \int_\mathcal{H} f(x) Q(dx)\\
    &= \int_\mathcal{H} \Big(f(x) - \frac{B}{2}\Big) (P - Q)(dx)\\
    &\leq \frac{B}{2} \int_\mathcal{H} |P - Q|(dx)\\
    &= B d_{\mathrm{TV}}(P, Q).
\end{align*}
Similarly, using the fact that $d_{\mathrm{TV}}(P, Q) = d_{\mathrm{TV}}(Q, P)$, we have 
\begin{equation*}
    \EX_{X\sim Q}f(X) - \EX_{X\sim P} f(X) \leq B d_{\mathrm{TV}}(P, Q), 
\end{equation*}
concluding the proof the lemma.
\end{proof}

The following lemma was proved in \cite[Theorem 1]{xu2020continuity}.

\begin{lemma}
\label{lemma:xu_prime}
Consider the bounded loss function $\ell: \mathcal{Y}\times\mathcal{Y} \to [0, B]$.  We have
\begin{align*}
    \Big|\inf_{f\in\mathcal{F}} \EX_{(X, Y)\sim P} \ell(f(X), Y) - \inf_{f\in\mathcal{F}} &\EX_{(X, Y)\sim Q} \ell(f(X), Y)\Big|\\ &\leq B d_{\mathrm{TV}}(P, Q), 
\end{align*}
in which $\mathcal{F}$ is the set of all measurable functions and it is assumed that the infimums are achieved.
\end{lemma}

\begin{proof}
Assume that the infimum under distribution $P$ and $Q$ are achieved at $f_P^*$ and $f_Q^*$ respectively. Hence, we can write
\begin{align*}
     \EX_{P} \ell(f_P^*(X), Y) &- \EX_{Q} \ell(f_Q^*(X), Y)\\
     &\leq \EX_{P} \ell(f_Q^*(X), Y) - \EX_{Q} \ell(f_Q^*(X), Y)\\
     &\leq B d_{\mathrm{TV}}(P, Q),
\end{align*}
where the last inequality follows from Lemma \ref{lemma:TV-continuity}. The lower bound can be proved similarly.
\end{proof}

Now, we will present the proof for Theorem \ref{theorem:minimax_general_compact_measures}.

\begin{proof}
Similar to the proof of Theorem \ref{theorem:minimax_used_with_xu2020}, we will prove that the sets $\Delta(\cW)$ and $\Delta(\cA)$, and the function $\Psi(P_W, P_A) = \EX_{A \sim P_A} \EX_{W\sim P_W} e(A, W)$
satisfy the conditions of Theorem \ref{theorem:general-minimax}.

\noindent \textbf{(Convexity)} Convexity of $\Delta(\cA)$ and $\Delta(\cW)$ hold trivially. 

\noindent \textbf{(Compactness)} Equip the set $\Delta(\cW)$ with the Prokhorov metric $d_P$ induced by $d_{\mathrm{TV}}$. By using Proposition 5.3  of \cite{jancol}, we know that the compactness of $(\cW, d_{\mathrm{TV}})$ implies that $(\Delta(\cW), d_P)$ is also compact.

\noindent \textbf{(Boundedness)} The boundedness of loss function implies the boundedness of the function $f$.

\noindent \textbf{(Continuity)} Fix any $P_A \in \Delta(\cA)$. By the application of bounded convergence theorem,  the continuity of $h(P_W) = \EX_{W\sim {P_W}} e(a, W)$ for any $a\in\cA$ is a sufficient condition for the continuity of $f(., P_A)$. The metric space $(\cW, d_{\mathrm{TV}})$ is compact, hence it is separable. Thus, using \cite[Theorem 4.2]{jancol} the convergence in $d_P$ is the same as weak convergence of measures. As a result, to show the continuity of $h(P_W)$, we should show that the function $e(a, .)$ is continuous.  For any $w, w' \in \cW$ we have
\begin{align*}
    \Big| &e(a, w) - e(a, w')\Big|\\
    &= \Big|\EX_{Z^nXY \sim w^{\otimes n+1}}[\ell(a(Z^n,X),Y)]\\
    &\hspace{3cm} - \EX_{Z^nXY \sim w'^{\otimes n+1}}[\ell(a(Z^n,X),Y)]\\
    &~\hspace{1cm}+ \EX_{Z \sim w'}[\ell(f_{w'}^*(X),Y)] - \EX_{Z \sim w}[\ell(f_{w}^*(X),Y)]\Big|.
\end{align*}
Hence, by the application of Lemma \ref{lemma:TV-continuity} and Lemma \ref{lemma:xu_prime} we have
\begin{align*}
    \Big| e(a, w) - e(a, w')\Big| &\leq B d_{\mathrm{TV}}(w, w')\\ &\hspace{1cm}+ B d_{\mathrm{TV}}(w^{\otimes n+1}, w'^{\otimes n+1}) \\
    &\leq B(n+2) d_{\mathrm{TV}}(w, w'),
\end{align*}
where the last inequality follows from the tensorization of Total Variation. So we have proved the continuity of $e(a, .)$; thus the continuity of $\Psi(., P_A)$ for any $P_A$ is proved.

We have shown that the function $\Psi(P_W, P_A)$ and the sets $\Delta(\cA)$ and $\Delta(\cW)$ satisfy the conditions of Theorem \ref{theorem:general-minimax}, which implies the minimax equality.
\end{proof}

\subsection{Proof of Theorem \ref{theorem:minimax_used_with_xu2020}}
Here, we present the proof for Theorem \ref{theorem:minimax_used_with_xu2020}.

\begin{proof}
We will prove that the sets $\Delta(\cW)$ and $\Delta(\cA)$, and the function
\begin{equation*}
    \begin{cases}
        \Psi: \Delta(\cW)\times \Delta(\cA) \to \mathbb{R}\\
        \Psi(P_W, P_A) = \EX_{A \sim P_A} \EX_{W\sim P_W} e(A, W)
    \end{cases}
\end{equation*}
satisfy the conditions of Theorem \ref{theorem:general-minimax}.

\noindent\textbf{(Convexity)} Convexity of $\Delta(\cA)$ and $\Delta(\cW)$ hold trivially. 

\noindent\textbf{(Compactness)} Equip the set $\Delta{(\cW)}$ with the Prokhorov metric $d_P$ induced by $d_\cW$. For an overview of the properties of the Prokhorov metric, see \cite{jancol}. From Proposition 5.3 of \cite{jancol}, we know that if $(\cW, d_\cW)$ is a compact metric space, then $(\Delta(\cW), d_P)$ is also a compact metric space.

\noindent\textbf{(Boundedness)} The boundedness of loss function implies the boundedness of the function $f$.

\noindent\textbf{(Continuity)} Fix any $P_A \in \Delta(\cA)$. By the application of bounded convergence theorem,  the continuity of $h(P_W) = \EX_{W\sim {P_W}} e(a, W)$ for any $a\in\cA$ is a sufficient condition for the continuity of $\Psi(., P_A)$. Note that $(\cW, d_\cW)$ is a compact metric space, hence separable. As a result, based on \cite[Theorem 4.2]{jancol}, convergence in the metric $d_P$ is the same as weak convergence of measures. Thus, to prove the continuity of $h(P_W)$, it is enough to show that for any $a\in\cA$, the function $e(a, .)$ is  bounded and continuous.  By definition,
\begin{align*}
    &e(a,w) =\EX^{w}_{Z^nXY}[\ell(a(Z^n,X),Y)]-\EX^{w}_{XY}[\ell(f^*_w(X),Y)]\\
    &= \int \psi(z^n; x, y) f(x, y | w) \prod_{i = 1}^{n}f(x_i, y_i | w)\;\; dz^n dx dy,
\end{align*}
where $\psi(z^n; x, y) = \ell(a(z^n,x),y) - \ell(f^*_w(x),y)$.
The loss function, the density functions, and the sets $\cX$ and $\cY$ are bounded and the densities are continuous in $w$ given any $z^n, x, y$; hence, the bounded convergence theorem implies the continuity of $e(a, .)$. Thus, $\Psi(P_W, P_A)$ is a continuous function of $P_W$ for any $P_A \in \Delta(\cA)$.

We have shown that the function $\Psi(P_W, P_A)$ and the sets $\Delta(\cA)$ and $\Delta(\cW)$ satisfy the conditions of Theorem \ref{theorem:general-minimax}, which implies the minimax equality.
\end{proof}

\subsection{Other minimax theorems}
\label{sec:wasserstein}
The Total Variation distance appeared very naturally in the proofs of Theorem \ref{theorem:minimax_used_with_xu2020} and Theorem \ref{theorem:minimax_general_compact_measures}. Under the Total Variation distance, the continuity conditions hold under very mild assumptions. However, in probability theory, compactness in Total Variation is not widely used. In this section, we replace Total Variation with the Wasserstein metric and prove a minimax theorem for unbounded loss functions under some Lipschitzness constraints. The following lemma will be key in proving this type of theorems.

\begin{lemma}
\label{lemma:liplem}
Let $f: \mathcal{H} \to \mathbb{R}$ be a real valued function and $P$ and $Q$ be two probability measures on $\mathcal{H}$. Assume that $f$ is $L$-Lipschitz. We have
\begin{equation*}
    |\EX_{X\sim P} f(X) - \EX_{X\sim Q} f(X)| \leq L d_{\mathrm{W}}(P, Q).
\end{equation*}
\end{lemma}
\begin{proof}
We can write
\begin{align*}
    |\EX_{X\sim P} &f(X) - \EX_{X\sim Q} f(X)|\\
    &\leq \sup_{g\in L-\text{Lipschitz}} |\EX_{X\sim P} g(X) - \EX_{X\sim Q} g(X)|\\
    &=L\sup_{g\in 1-\text{Lipschitz}} |\EX_{X\sim P} g(X) - \EX_{X\sim Q} g(X)|\\
    &= L d_{\mathrm{W}}(P, Q),
\end{align*}
which concludes the proof.
\end{proof}

Using this lemma, we can prove the following minmax theorem.
\begin{theorem}
\label{theorem:minimax_general_compact_measures_wasserstein}
Let $(\mathcal{W}, d_{\mathrm{W}})$ be a compact metric space of measures on $(\mathcal{Z}, d_{\mathcal{Z}})$ where $d_{\mathrm{W}}$ is the Wasserstein's distance, $\Delta(\mathcal{W})$ be the set of all Borel probability measures on $\mathcal{W}$, $\Delta(\cA)$ be the set of all probability measures on the space of algorithms, and $\ell: \mathcal{Y}  \times \mathcal{Y} \to \mathbb{R}$ be a loss function. Assume that:

\noindent(1) For any $P_A \in \Delta(\cA)$, the function $$\Gamma(z^n, z) = \EX_{A\sim P_A} l\big(A(x; z^n), y\big)$$ is Lipschitz. 

\noindent(2) For any $w\in \cW$, the function $$\Gamma^*_w (z) = l\big(f_w^*(x), y\big)$$ is Lipschitz.

\noindent(3) For any $w \in \cW$ and any $P_A \in \Delta(\cA)$ we have $\EX_{A\sim P_A}e(A, w) \leq c$.
\vspace{0.2cm}

Under these conditions,  we have
\begin{align*}
&\adjustlimits \sup_{P_W \in \Delta(\cW)} \inf_{P_A \in \Delta(\cA)}  \EX_{\alg W}  [e(\alg,W)]\\
&\hspace{2.5cm}=
\adjustlimits \inf_{P_A \in \Delta(\cA)} \sup_{P_W \in \Delta(\cW)} \EX_{\alg W} [e(\alg,W)].
\end{align*}
\end{theorem}
\begin{proof}
Similar to previous theorems, we will prove that the sets $\Delta(\cW)$ and $\Delta(\cA)$, and the function
$\Psi(P_W, P_A) = \EX_{A \sim P_A} \EX_{W\sim P_W} e(A, W)$
satisfy the conditions of Theorem \ref{theorem:general-minimax}.

\noindent \textbf{(Convexity)} Convexity of $\Delta(\cA)$ and $\Delta(\cW)$ hold trivially. 

\noindent \textbf{(Compactness)} Equip the set $\Delta(\cW)$ with the Prokhorov metric $d_P$ induced by $d_{\mathrm{W}}$. By using Proposition 5.3  of \cite{jancol}, the metric space $(\Delta(\cW), d_P)$ is compact.

\noindent \textbf{(Boundedness)} For any distribution $w \in \cW$ and any $P_A \in \Delta(\cA)$ we have $\EX_{P_A}e(A, w) \leq c$. Hence, $\Psi(P_W, P_A) \leq C$.

\noindent \textbf{(Continuity)} Fix $P_A \in \Delta(\cA)$. The metric space $(\cW, d_{\mathrm{W}})$ is compact, hence it is separable. Thus, using \cite[Theorem 4.2]{jancol} the convergence in $d_P$ is the same as weak convergence of measures. As a result, boundedness and continuity of $h(w) = \EX_{A\sim P_A} e(A, w)$ implies the continuity of $\Psi(P_W, P_A)$. For any $w_1, w_2\in \Delta(\cW)$, we have
\begin{align*}
    \Big|&h(w_1) - h(w_2)\Big| =  \Big|\EX_{A\sim P_A} e(A, w_1) - \EX_{A\sim P_A} e(A, w_2)\Big|\\
    &\leq \Big|\EX_{Z^n Z \sim w_1^{\otimes n+1}}\Big[\EX_{A\sim P_A} \ell\big(A(X; Z^n), Y\big)\Big]\\ 
    &\hspace{1.5cm}- \EX_{Z^n Z \sim w_2^{\otimes n+1}}\Big[\EX_{A\sim P_A}\ell\big(A(X; Z^n), Y\big)\Big]\Big|\\
    &~\hspace{.5cm}+ \Big|\EX_{Z \sim w_1} \Big[\ell\big(f_{w_1}^*(X), Y\big)\Big] - \EX_{Z \sim w_2}\Big[\ell\big(f_{w_2}^*(X), Y\big)\Big]\Big|\\
    &\leq d_{\mathrm{W}}(w_1^{\otimes n+1}, w_2^{\otimes n+1})\\ 
    &\hspace{0.5cm}+ \Big|\EX_{Z \sim w_1} \Big[\ell\big(f_{w_1}^*(X), Y\big)\Big] - \EX_{Z \sim w_2}\Big[ \ell\big(f_{w_2}^*(X), Y\big)\Big]\Big|\\
    &\leq d_{\mathrm{W}}(w_1^{\otimes n+1}, w_2^{\otimes n+1})\\
    &\hspace{0.5cm}+ \Big|\EX_{Z \sim w_1} \Big[\ell\big(f_{w_2}^*(X), Y\big)\Big] - \EX_{Z \sim w_2}\Big[ \ell\big(f_{w_2}^*(X), Y\big)\Big]\Big|\\
    &\leq d_{\mathrm{W}}(w_1^{\otimes n+1}, w_2^{\otimes n+1}) + d_{\mathrm{W}}(w_1, w_2),
\end{align*}
where the second inequality follows from the Lipschitzness of $\Gamma(z, z^n) = \EX_{A\sim P_A} \ell\big(A(x; z^n), y\big)$ for every $P_A$ and the application of Lemma \ref{lemma:liplem}. The last inequality follows from the Lipschitzness of $\Gamma^*_w(z) = l(f_w^*(x), y)$ for any $w\in \cW$ and Lemma \ref{lemma:liplem}. This inequality alongside the Tensorization of Wasserstein distance, proves the continuity of $h(w)$; Hence the continuity of $\Psi(P_W, P_A)$.

We have shown that the function $\Psi(P_W, P_A)$ and the sets $\Delta(\cA)$ and $\Delta(\cW)$ satisfy the conditions of Theorem \ref{theorem:general-minimax}, which concludes the proof.
\end{proof}

\subsection{Proof of Theorem \ref{theorem:compactness_W_VC}}
\label{When_do_minimax_theorems_work}

To prove Theorem \ref{theorem:compactness_W_VC}, we first need to prove the following lemma.
\begin{lemma}
\label{lemma:TV_upper_bound_VC}
For any $\mathbb{P}_{P_X^1, f^1}$ and $\mathbb{P}_{P_X^2, f^2}$ in $\mathcal{W}$, we have
\begin{align*}
    d_{\mathrm{TV}}&(\mathbb{P}_{P_X^1, f^1}, \mathbb{P}_{P_X^2, f^2})\\
    &\leq d_{\mathrm{TV}} (P_X^1, P_X^2) + P_{X}^1(f^1\neq f^2) + P_{X}^2(f^1\neq f^2).
\end{align*}
\end{lemma}
\begin{proof}
Given $f^1, f^2 \in \mathcal{F}$, any $A \subseteq \mathcal{Z}$ can be written as 
\begin{equation}
    A = A^{++} \cup A^{+-}\cup A^{-+},
\end{equation}
in which
\begin{align*}
    A^{++} &= \{(x, y) \in A \; | \; y = f^1(x) \text{ and } y = f^2(x)\}\\
    A^{+-} &= \{(x, y) \in A \; | \; y = f^1(x) \text{ and } y \neq f^2(x)\}\\
    A^{-+} &= \{(x, y) \in A \; | \; y \neq f^1(x) \text{ and }  y = f^2(x)\}
\end{align*}
are disjoint sets. We have
\begin{align*}
    d_{\mathrm{TV}}&(\mathbb{P}_{P_X^1, f^1}, \mathbb{P}_{P_X^2, f^2}) = \sup_{A \subseteq \mathcal{Z}} \big|\mathbb{P}_{P_X^1, f^1}(A) - \mathbb{P}_{P_X^2, f^2}(A)\big|\\
    &\leq \sup_{A \subseteq \mathcal{Z}} \big|\mathbb{P}_{P_X^1, f^1}(A^{++}) - \mathbb{P}_{P_X^2, f^2}(A^{++})\big|\\
    &~\hspace{1cm}+\sup_{A \subseteq \mathcal{Z}} \big|\mathbb{P}_{P_X^1, f^1}(A^{+-}) - \mathbb{P}_{P_X^2, f^2}(A^{+-})\big|\\
    &~\hspace{1cm}+\sup_{A \subseteq \mathcal{Z}} \big|\mathbb{P}_{P_X^1, f^1}(A^{-+}) - \mathbb{P}_{P_X^2, f^2}(A^{-+})\big|\\
    &\leq d_{\mathrm{TV}}(P_X^1, P_X^2) + P_X^1(f_1 \neq f_2) + P_X^2(f_1 \neq f_2),
\end{align*}
where the last inequality follows from the fact that
\begin{align*}
    \sup_{A \subseteq \mathcal{Z}} \big|&\mathbb{P}_{P_X^1, f^1}(A^{++}) - \mathbb{P}_{P_X^2, f^2}(A^{++})\big|\\
    &= \sup_{A \subseteq \mathcal{Z}^{++}} \big|\mathbb{P}_{P_X^1, f^1}(A^{++}) - \mathbb{P}_{P_X^2, f^2}(A^{++})\big|\\
    &\leq d_{\mathrm{TV}}(P_X^1, P_X^2),
\end{align*}
and that $\mathbb{P}_{P_X^1, f^1}(A^{-+}) = \mathbb{P}_{P_X^2, f^2}(A^{+-}) = 0$.
\end{proof}

We are now ready to prove Theorem \ref{theorem:compactness_W_VC}.

\begin{proof}
Consider the mapping $\Psi: \mathcal{P}\times \mathcal{F} \to \mathcal{W}$ defined as $\Psi(P_X, f) = \mathbb{P}_{P_X, f}$.
We will prove that this mapping is continuous: 

Consider an arbitrary sequence $\{(P_X^i, f^i)\}_{i = 1}^{\infty}$ converging to $(P_X^*, f^*)$; i.e.,
\begin{equation*}
\lim_{i \to \infty}\max\Big[d_\mathrm{TV}(P_X^i, P_X^*), ||f^i - f^*||_{L^2}\Big] = 0.    
\end{equation*}
Note that 
\begin{align*}
P_X^i(f^i \neq f^*) &= \int_{\mathcal{X}} |f^i(x) - f^*(x)|^2 p_X^i(x)\; dx\\
&\leq c \int_{\mathcal{X}} |f^i(x) - f^*(x)|^2 dx \to 0,
\end{align*}
and similarly, $P_X^*(f^i \neq f^*) \to 0$. Hence, by using Lemma \ref{lemma:TV_upper_bound_VC}, we have $d_{\mathrm{TV}}(\mathbb{P}_{P_X^i, f^i}, \mathbb{P}_{P_X^*, f^*}) \to 0$, proving the continuity. The continuous image of a compact set is compact; hence, $\mathcal{W}$ is compact in Total Variation.
\end{proof}

\subsection{Proof of Theorem \ref{theorem:TV_compactness_general}}

\begin{proof}
First, we will show that the the mapping  $\Psi(\gamma) = P_\gamma$ is continuous. We have
\begin{align*}
    \lim_{\gamma \to \gamma^*} d_{\mathrm{TV}}(P_{\gamma} , P_{\gamma^*}) = \lim_{\gamma \to \gamma^*} \int_{\mathbb{R}^n} \big|f(x|\gamma) - f(x|\gamma^*) \big| \; dx.
\end{align*}
Note that for any $x\in \mathbb{R}^n$, we have $|f(x|\gamma) - f(x|\gamma^*)| \leq g(x)$. Thus, using dominated convergence theorem we have
\begin{equation*}
    \lim_{\gamma \to \gamma^*} d_{\mathrm{TV}}(P_{\gamma}, P_{\gamma^*}) = 0,
\end{equation*}
proving the continuity of $\Psi$. Using the fact that the continuous image of a compact set is also compact, the compactness of $\cW$ is proved.
\end{proof}

\subsection{Gaussian Parametric Family}
\label{sec:gaussian_proof}
In this section, we will show that the parametric family of $d$-dimensional normal distributions $$\mathcal{P} = \Big\{\mathcal{N}(0, \Sigma) \Big| \Sigma_H \succcurlyeq \Sigma \succcurlyeq \Sigma_L \Big\},$$
where $\det(\Sigma_L) > \epsilon$ satisfies the conditions of Lemma \ref{theorem:TV_compactness_general}.

\noindent\textbf{(Integrable Upper Bound)} Let $\mathcal{N}(0, \Sigma_1), \mathcal{N}(0, \Sigma_2) \in \mathcal{P}$. We have
\begin{align*}
    \Lambda(\Sigma_1, &\Sigma_2) = \Big|(2\pi)^{-\frac{d}{2}}(\det \Sigma_1)^{-\frac{1}{2}}\exp\Big(\frac{-x^\top \Sigma_1 x}{2}\Big)\\
    &\hspace{1.5cm}- (2\pi)^{-\frac{d}{2}}(\det \Sigma_2)^{-\frac{1}{2}}\exp\Big(\frac{-x^\top \Sigma_2 x}{2}\Big)\Big|\\
    &\leq (2\pi)^{-\frac{d}{2}}\epsilon^{-\frac{1}{2}}\Big|\exp\Big(\frac{-x^\top \Sigma_1 x}{2}\Big) - \exp\Big(\frac{-x^\top \Sigma_2 x}{2}\Big)\Big|\\
    &\leq 2(2\pi)^{-\frac{d}{2}}\epsilon^{-\frac{1}{2}}\Big|\exp\Big(\frac{-x^\top \Sigma_L x}{2}\Big)\Big|,
\end{align*}
which is an integrable upper bound.

\noindent\textbf{(Continuity of Density)} Consider the space of $d\times d$ matrices with Frobenius norm. $\det(.)$ is a continuous function in Frobenius norm; hence, the density of the multivariate normal is continuous in the parameter $\Sigma$, given the fact that $\det(\Sigma) > 0$.

\noindent\textbf{(Compactness of Parameters)} The set $ \{\Sigma\in \mathbb{R}^{d\times d} \; | \; \Sigma_H \succcurlyeq \Sigma \succcurlyeq \Sigma_L\}$ is a bounded and closed subset of $\mathbb{R}^{d\times d}$; hence, it is compact in Frobenius norm.

Thus, the set $\mathcal{P}$ satisfies the conditions of Lemma \ref{theorem:TV_compactness_general} implying the compactness of $\mathcal{P}$.

\subsection{Proof of Theorem \ref{theorem:info_bounds_after_minimax}}
\begin{proof}
We have 
\begin{align*}
\adjustlimits \inf_{P_\cA \in \Delta(\cA)} \sup_{P_W \in \Delta(\cW)} &\EX_{\alg,W} [e(\alg,W)]\\
=&\adjustlimits \sup_{P_W \in \Delta(\cW)} \inf_{P_\cA \in \Delta(\cA)} \EX_{\alg,W} [e(\alg,W)]
\\
\le& 2B  g\Big(\sup_{P_W \in \Delta(\cW)} I(W;Y|X,Z^n)\Big)
\\
\le& 
g\Big(\frac{2B\kappa_n}{n}\Big),
\end{align*}
in which $g:\mathbb{R}^+\to\mathbb{R}^+$ is either identity function $g(x)=x$ for case (a) or square root $g(x)=\sqrt{x}$ for case (b). The first equality is based on Theorem~\ref{theorem:minimax_used_with_xu2020}. The next two inequalities are based on Theorem~\ref{theorem:xu2020_main} and the fact that $g$ is an increasing function. A similar approach can be used when $g(x)=\varphi^{*-1}(x)$ is inverse of the cumulant generating function of the random variable of interest.
\end{proof}

\subsection{Lemma~\ref{lemma:fisher} and its conditions}

In this section, we first state the conditions of Lemma \ref{lemma:fisher}. These conditions can also be found in Section 2 of \cite{clarke1994jeffrey}:

Let $\mathcal{W} \subseteq \mathbb{R}^d$ and assume that $P_{Z}^{W}(.|w)$ is absolutely continuous for all $w \in \cW$.

\noindent\textbf{(Conditions of $\mathcal{W}$)} The parameter space $\mathcal{W}$ has a non-void interior and its boundary has a $d$-dimensional Lebesgue measure zero.

\noindent\textbf{(Smoothness)} For almost every $z$, the density $p_{Z}^{W}(z|w)$ is twice continuously differentiable in $w$. There also exists $\delta(w)$ such that for every indices $1 \leq j, k \leq d$:
	\begin{equation*}
		f(w) = \EX \Bigg[\sup_{w': ||w' - w||\leq \delta(w)}\Big| \frac{\partial^2}{\partial w_j' \partial w_k'} \log p_{Z}^{W}(Z|w')\Big|\Bigg]
	\end{equation*}
	is finite and continuous.
	
\noindent\textbf{(Existence of Moments)} There exists $\zeta>0$ such that for each $1 \leq j \leq d$, the moment
	\begin{equation*}
		\EX \Big[\Big| \frac{\partial}{\partial w_j} \log p_{Z}^{W}(Z|w)\Big|^{2+\zeta}\Big]
	\end{equation*}
	is finite and continuous, as a function of $w$.
	
\noindent\textbf{(Fisher Information Condition)} The Fisher information matrix, and the second derivative of KL divergence
	\begin{equation*}
		[I(w)]_{j, k} = \EX \Big[\frac{\partial}{\partial w_j} \log p_{Z}^{W}(Z|w) \frac{\partial}{\partial w_k} \log p_{Z}^{W}(Z|w)\Big],
	\end{equation*}
	and 
	\begin{equation*}
	[J(w)]_{j, k} = \Big[\frac{\partial^2}{\partial w'_j\partial w'_k} \mathrm{KL} \Big(P_{Z}^{w} || P_{Z}^{w'}\Big)\Big|_{w' = w}\Big];
	\end{equation*}	
	are equal and positive-definite.
	
\noindent\textbf{(One to One)} For $w\neq w'$, we have $P_{Z}^{w} \neq P_{Z}^{w'}$.
	
\noindent\textbf{(Continuity)} The prior on $W$ is assumed to be continuous and is supported on a compact subset of the interior of $\mathcal{W}$.

Under these conditions, as $n\to \infty$ we have 
\begin{equation*}
	I(W; Z^n) = \frac{d}{2} \log\left(\frac{n}{2\pi e}\right) + h(W) + \EX \big[\log| J_{Z}^W (W)|\big] + o(1).
\end{equation*}
To prove \eqref{eq:fisher_tighter}, note that $$I(W; Z^{n+1}) = I(W;Z^n) + I(W; Z|Z^n).$$ Thus, by application of \cite[Lemma 6]{haussler1995general}, we have $I(W; Z|Z^n) = O(\frac{d}{n})$. Hence,
\begin{equation*}
    I(W; Y|X, Z^n) \leq I(W; Z|Z^n) = O\left(\frac{d}{n}\right),
\end{equation*}
where the first inequality follows from chain rule.

\subsection{Individual upper rates}

\begin{proof}
Under the conditions of Theorem~\ref{theorem:minimax_used_with_xu2020} and Lemma~\ref{lemma:fisher} for 
$\lim_{n\to \infty} \EX[\ell(\alg(Z^n,X),Y)]-\EX[\ell(f^*_W(X),Y)]$,
one can change the order of $\sup$ and $\inf$; i.e.,
\begin{align*}
v &= \inf_{\{P_{\cA}^n\}_{n\in\mathbb{N}}} \sup_{P_W\in \Delta(\cW)}\\
& \hspace{1.5cm} \limsup_{n\to +\infty} \frac{\EX[\ell(\alg(Z^n,X),Y)]-\EX[\ell(f^*_W(X),Y)]}{b_n}\\
&= \sup_{P_W\in \Delta(\cW)} \inf_{\{P_{\cA}^n\}_{n\in\mathbb{N}}} \\
& \hspace{1.5cm} \limsup_{n\to +\infty} \frac{\EX[\ell(\alg(Z^n,X),Y)]-\EX[\ell(f^*_W(X),Y)]}{b_n}\\
&= \sup_{P_W} \limsup_{n\to +\infty} \inf_{P_{\cA}^n}  \frac{\EX[\ell(\alg(Z^n,X),Y)]-\EX[\ell(f^*_W(X),Y)]}{b_n}.
\end{align*}
Note that for each $n$, the minimum excess risk is $O(\frac{1}{n})$ for quadratic loss and $O(\frac{1}{\sqrt{n}})$ for general bounded loss. Hence, the mentioned individual upper rates hold.
\end{proof}

\begin{remark}
Note that in Theorem \ref{theorem:individual_upper_rate}, it is assumed that the function $\lim_{n\to \infty} \EX[\ell(\alg(Z^n,X),Y)]-\EX[\ell(f^*_W(X),Y)]$ satisfies the conditions of minimax equality. This assumption holds under various conditions. For example, assume that $(\Delta(\cW), d_{\mathrm{TV}})$ is compact. Given $\{P_A^n\}_{n = 1}^{\infty}$, consider the sequence of functionals 
\begin{equation*}
    \begin{cases}
    \Gamma_n: \Delta(\cW) \to \mathbb{R}\\
    \Gamma_n(P_W) = \frac{1}{b_n}\EX_{W\sim P_W} \EX_{\alg \sim P_A^n} e(A, W).
    \end{cases}
\end{equation*}
This functional is linear in $P_W$ and $\Delta(\cW)$ is a normed space. As a result of boundedness of the loss function and Lemma \ref{lemma:xu_prime}, for any $n$, the functional $\Gamma_n(P_W)$ is continuous. By application of Uniform boundedness principle, if the limit of  $\Gamma_n(P_W)$ exists for all $P_W \in \cW$ (possibly infinite), then these pointwise limits define a continuous operator $\Gamma$. Hence, the minimax equality holds.
\end{remark}

\subsection{Proof of Theorem~\ref{thm:posterior_Samlping}}
First, a general result for the bound on mutual information based on the inverse of the cumulant generating function is presented. This results can be found in \cite{xu2020minimum}. %

\begin{lemma}
\label{lemma:base_bound_based_on_cumulant_generating_function}
Consider distributions $P$ and $Q$ on a measurable set $\mathbb{U}$ and a function $f:U\to \mathbb{R}$. Suppose that the cumulant generating function of the random variable $f(U)$ under distribution $Q$ exists for $\lambda \in [0, b)$ and is upper bounded by $\varphi(\lambda)$; i.e., 
$$
\log \EX_Q\big[e^{-\lambda (f(U)-\EX_Q[f(U)])}\big]\le \varphi (\lambda).
$$
Then 
$$
\EX_Q[f(U)]-\EX_P[f(U)] \le \varphi^{*-1} (\KLdiv{P}{Q}),
$$
where
$$
\varphi^*(\gamma)=\sup_{\lambda \in [0,b)} \lambda\gamma -\varphi(\lambda),\; \gamma\in \mathbb{R},
$$
is the Legendre dual of $\varphi$ and $\varphi^{*-1}$ is the generalized-inverse of $\varphi^*$ defined as
$$
\varphi^{*-1}(x) =\sup \{\gamma \in \mathbb{R}:\varphi^*(\gamma)\le x,\; x\in \mathbb{R}\}.
$$
\end{lemma}

The following lemma will be eventually useful in the proof of Theorem~\ref{thm:posterior_Samlping}. This is an extension of the inequality $I(W;Y|XZ^n)\le \frac{1}{n} I(W;Z^n)$ used by \cite{xu2020minimum}.

\begin{lemma}
\label{lemma:general_bound_for_1/n}
Consider two distributions $P_{WZ^n}=P_{W} \otimes (P_{Z}^W)^{\otimes n}$ and  $P_{WZ^n}'=P_{0} \otimes (P_{Z}^W)^{\otimes n}$. We have 
\begin{equation*}
\DKL{P_W^{Z^n}}{P_W^{'Z^n}} \le \frac{1}{n} \DKL{P_{WZ^n}}{P'_{WZ^n}},
\end{equation*}
where we used the convention that  $$\DKL{P^{Z^n}_W}{P^{'Z^n}_W}\triangleq\EX_{Z^n\sim P}[\KLdiv{P^{Z^n}_W}{P^{'Z^n}_W}].$$
\end{lemma}
\begin{proof}

We first show that $$\DKL{P_{WZ_i}^{Z^{i-1}}}{P_{WZ_i}^{'Z^{i-1}}}\ge \DKL{P_{WZ_{i+1}}^{Z^{i}}}{P_{WZ_{i+1}}^{'Z^{i}}},$$ for all $1\le i \le n-1$. Using chain rule and noticing that
$P_{Z_i}^W=P_{Z_j}^W=P_{Z_i}^{'W}=P_{Z_j}^{'W}$ for all $1\le i,j\le n$, we have
\begin{align*}
&\DKL{P_{WZ_i}^{Z^{i-1}}}{P_{WZ_i}^{'Z^{i-1}}} - \DKL{P_{WZ_{i+1}}^{Z^{i}}}{P_{WZ_{i+1}}^{'Z^{i}}}   
\\
&\;\;
=\DKL{P_{Z_i}^{WZ^{i-1}}}{P_{Z_i}^{'WZ^{i-1}}} +\DKL{P_{W}^{Z^{i-1}}}{P_{W}^{'Z^{i-1}}}\\
&\hspace{2cm} - \DKL{P_{W}^{Z^{i}}}{P_{W}^{'Z^{i}}}
- \DKL{P_{Z_{i+1}}^{WZ^{i}}}{P_{Z_{i+1}}^{'WZ^{i}}}
\\
&\;\;
=\DKL{P_{W}^{Z^{i-1}}}{P_{W}^{'Z^{i-1}}}
- \DKL{P_{W}^{Z^{i}}}{P_{W}^{'Z^{i}}}
\\
&\;\;
=\DKL{P_{W}^{Z^{i-1}}}{P_{W}^{'Z^{i-1}}}
- \DKL{P_{Z_i}^{WZ^{i-1}}}{P_{Z_i}^{'WZ^{i-1}}}\\
&\hspace{2cm}- \DKL{P_{W}^{Z^{i-1}}}{P_{W}^{'Z^{i-1}}}
+ \DKL{P_{Z_i}^{Z^{i-1}}}{P_{Z_i}^{'Z^{i-1}}}
\\
&\;\;
=\DKL{P_{Z_i}^{Z^{i-1}}}{P_{Z_i}^{'Z^{i-1}}}
\\
&\;\;
\ge 0.
\end{align*}
Using the chain rule, we have
\begin{align*}
\DKL{P_{WZ^n}}{P'_{WZ^n}} &= \sum_{i=1}^n \DKL{P_{WZ_i}^{Z^{i-1}}}{P_{WZ_i}^{'Z^{i-1}}}
\\
&\hspace{-1cm}=n \DKL{P_{WZ_n}^{Z^{n-1}}}{P_{WZ_n}^{'Z^{n-1}}}
\\
&\hspace{-1cm}=n\DKL{P_{W}^{Z^n}}{P_{W}^{'Z^n}} + \DKL{P_{Z_n}^{Z^{n-1}}}{P_{Z_n}^{'Z^{n-1}}}
\\
&\hspace{-1cm}\ge n\DKL{P_{W}^{Z^n}}{P_{W}^{'Z^n}}
,
\end{align*}
which concludes the proof.
\end{proof}

Now we are ready to prove Theorem~\ref{thm:posterior_Samlping}.

\begin{proof}
Let $f^*_w(x)$ be the optimal decision assuming that $w$ is the true parameter. Let $P_{WXYZ^n}=P_{W}\otimes P_{XYZ^n}^W$ be the true distribution and  $P'_{WXYZ^n}=P_{0} \otimes P_{XYZ^n}^W$ be the distribution used by the posterior sampling algorithm according to $P_0$.
Now fix $X=x$ and $Z^n=z^n$. Now that the renadomness is only on $W$ and $Y$, we will use Lemma~\ref{lemma:base_bound_based_on_cumulant_generating_function} to control the excess risk for this particular $x$ and $z^n$. To do that let $P=P_{WY}^{xz^n}$ and  $Q=P_{W}^{'xz^n}\otimes P_Y^{xz^n}$. Note that 
in the definition of $Q$, while $W$ is generated from wrong distribution $P'$, the output $Y$ is still generated from the correct distribution. Thus, it is possible to calculate the Bayesian excess risk for this particular $xz^n$ by
\begin{align*}
\EX_Y^{xz^n}&[\ell(\alg_{\text{PS}(P_0)}(z^n,x), Y)]-\EX_{WY}^{xz^n}[\ell(f^*_W(x), Y)] 
\nonumber
\\
&=\EX_{WY\sim Q}[\ell(f^*_W(x),Y)] - \EX_{WY\sim P}[\ell(f^*_W(x),Y)]
\nonumber
\\
&\le \varphi^{*-1}(\KLdiv{P}{Q})
\\
&=\varphi^{*-1}\left(\EX\left[\log \frac{P^{xz^n}(WY)}{P^{'xz^n}(W)P^{xz^n}(Y)}\right]\right)
\nonumber
\\
&=\varphi^{*-1}\left(\EX\left[\log \frac{P^{xz^n}(WY)}{P^{xz^n}(W)P^{xz^n}(Y)}\times \frac{P^{xz^n}(W)}{P^{'xz^n}(W)}\right]\right)
\nonumber
\\
&= \varphi^{*-1}(I^{xz^n}(W;Y) + \KLdiv{P^{xz^n}_W}{P^{'xz^n}_W}).
\nonumber
\end{align*}
To prove the inequality, Lemma~\ref{lemma:base_bound_based_on_cumulant_generating_function} is used by noting that the conditions are assumed to hold for any $x$, $z^n$ and $P_W$. Recall that $\varphi^{*-1}$ is a concave function. By taking expectation with respect to $X$ and $Z^n$ and using Jensen's inequality, we have
\begin{align}
\EX [&e(\alg_{\text{PS}(P_0)},W)]\nonumber\\
&\le \varphi^{*-1}( I(W;Y|XZ^n) + \DKL{P^{XZ^n}_W}{P^{'XZ^n}_W})
\nonumber
\\
&\le\varphi^{*-1}(I(W;Y|XZ^n) + \DKL{P^{Z^n}_W}{P^{'Z^n}_W})
\label{eq:tmp1}
\\
&\le
\varphi^{*-1}\left(I(W;Y|XZ^n) + \frac{1}{n}\DKL{P_W}{P'_W}\right)
\label{eq:tmp2}
.
\end{align}
where we used the convention that 
$$\DKL{P^{U}_V}{P^{'U}_V}\triangleq\EX_{U\sim P}[\KLdiv{P^{U}_V}{P^{'U}_V}]$$ for random variables $V$ and $U$.
Inequality \eqref{eq:tmp1} is due to %
\begin{align*}
&\DKL{P_W^{XZ^n}}{P_W^{'XZ^n}}\\
&=\DKL{P_{WX}^{Z^n}}{P_{WX}^{'Z^n}}-\DKL{P_{X}^{Z^n}}{P_{X}^{'Z^n}}
\\
&=
\DKL{P_{W}^{Z^n}}{P_{W}^{'Z^n}}+\DKL{P_{X}^{W}}{P_{X}^{'W}}-\DKL{P_{X}^{Z^n}}{P_{X}^{'Z^n}}
\\
&=
\DKL{P_{W}^{Z^n}}{P_{W}^{'Z^n}}-\DKL{P_{X}^{Z^n}}{P_{X}^{'Z^n}}
\\
&\le
\DKL{P_{W}^{Z^n}}{P_{W}^{'Z^n}},
\end{align*}
and \eqref{eq:tmp2} is based on Lemma~\ref{lemma:general_bound_for_1/n}.
Finally, the proof is concluded by taking supremum with respect to $P_W$. We have
\begin{align}
&\sup_{P_W\in\mathcal{M}_W}\nonumber \EX_{W\sim P_W} [e(\alg_{\text{PS}(P_0)},W)] \\
&\hspace{0.5cm}\le \varphi^{*-1}\Big( \sup_{P_W\in\mathcal{M}_W} I(W;Y|XZ^n) + \frac{1}{n}\DKL{P_W}{P'_W}\Big),
\nonumber
\\
\label{eq:distribution_shift_bound_v2}
&\hspace{0.5cm}=
\varphi^{*-1}\Big( \sup_{P_W\in\mathcal{M}_W} I(W;Y|XZ^n) + \frac{r}{n}\Big)
\nonumber
\end{align}
where the last equality is derived by noting that $P'_W=P_0$ is assumed to achieve
$r=\sup_{P_W\in\mathcal{M}_W} \KLdiv{P_W}{P_0}$. 
\end{proof}

\subsection{Proof of Theorem \ref{RLS:MI}}

\begin{proof} For the described problem, we have
\begin{align*}
    &p(w|z^n) \propto p(z^n|w)p(w)\\
    &= p(w) \prod_{i = 1}^{n} p(z_i|w)\\
    &\propto \exp\bigg(\frac{-||w-\mu||^2}{2\sigma_w^2}\bigg)\prod_{i = 1}^{n} \exp\Bigg(\frac{-(w^\top \phi(x_i) - y_i)^2}{2\sigma_e^2}\Bigg)\\
    &\propto \exp\Bigg(-\frac{1}{2}\bigg[ w^\top\Big(\frac{\boldsymbol{\Phi}^\top\boldsymbol{\Phi}}{\sigma_e^2} + \frac{I}{\sigma_w^2}\Big)w  - 2w^\top \Big(\frac{\boldsymbol{\Phi}^\top\boldsymbol{Y}}{\sigma_e^2} + \frac{\mu}{\sigma_w^2}\Big) \bigg] \Bigg),
\end{align*}
where the third line follows from the fact that $W \independent{} X$. This proves the first part of the theorem.  Using this, we can calculate the mutual information. Note that for two multivariate Normal distributions we have
\begin{align*}
    \KLdiv{\mathcal{N}(\mu_1, \Sigma_1)}{\mathcal{N}(\mu_2, \Sigma_2)}=&\frac{1}{2}\biggl[\log\frac{|\Sigma_2|}{|\Sigma_1|} - d +\{ \Sigma_2^{-1}\Sigma_1 \}\\
     &  + (\mu_2 - \mu_1)^\top \Sigma_2^{-1}(\mu_2 - \mu_1)\biggl]. 
\end{align*}
Note that
\begin{equation*}
    \Sigma_W^{-1}\Sigma_W^{Z^n} = \lambda \Big(\frac{\boldsymbol{\Phi}^\top\boldsymbol{\Phi}}{\sigma_e^2}+\lambda I\Big)^{-1} \leadsto \text{ eigenvalues } = \frac{\lambda}{\lambda + n\hat\sigma_i}.
\end{equation*}
Thus, we have
    \begin{align*}
        I(W; Z^n) &= \EX_{Z^n} \Big[ \KLdiv{P_W^{Z^n}}{P_W} \Big]\\
        &= \EX_{Z^n}\biggl[\log\frac{\sigma^{2d}_w}{\prod_{i=1}^d \frac{\sigma_e^2}{n\hat{\sigma}_i+\lambda}} + \sum_{i=1}^d \frac{\lambda}{n\hat{\sigma}_i+\lambda} -d \\
        & \hspace{4.25cm}+ \frac{1}{\sigma_w^2}\norm{\mu_W^{Z^n} - \mu}^2 \biggl],
    \end{align*}
concluding the proof.
\end{proof}

\subsection{Proof of Theorem \ref{thm:posterior_sampling_RLS}}
\label{sec:theorem-10}

\begin{proof}
Fix $P_W \sim \mathcal{N}(\mu, \sigma_w^2 I)$. To use Theorem \ref{thm:posterior_Samlping}, we will need to upper bound the cumulant generating function of $$\ell(A_{PS(P_0)}(z^n, x), Y) = (W_{PS(P_0)}^\top \phi(x) - W^\top \phi(x) - E)^2 \triangleq H^2.$$
Note that
\begin{align*}
    \EX^{z^n x}&[H^2]\\ &= \phi^\top(x) \Big(\EX^{z^n x}\big[(W_{\text{PS}}-W)(W_{\text{PS}}-W)^\top\big]\Big)\phi(x) + \sigma_e^2\\
    &= \phi^\top(x) \Big( \Sigma_W^{z^n } + \Sigma_W^{z^n } \Big)\phi(x) + \sigma_e^2.
\end{align*}
Thus, the variance of $H$ can be upper bounded as
\begin{align*}
    \mathrm{Var}^{z^n x}(H) &\leq \EX^{z^n x}[H^2]\\& \leq 2 \sigma_{\max}\, (\Sigma_W^{z^n}) ||\phi(x)||_2^2 + \sigma_e^2\\
    &\leq 2\sigma_w^2 \, ||\phi(x)||^2 + \sigma_e^2\\
    &\leq 2\sigma_w^2 + \sigma_e^2.
\end{align*}
The random variable $H$ is normal; thus, it can be seen as a $(2\sigma_w^2 + \sigma_e^2)$-subgaussian random variable. This implies that $$L = \ell(A_{\text{PS}}(z^n, x), Y)$$ is $(16(2\sigma_w^2+\sigma_e^2), 16(2\sigma_w^2+\sigma_e^2))$-subexponential for all $x, z^n$; i.e., for all $|\lambda| \leq \frac{1}{32\sigma_w^2 + 16 \sigma_e^2}$,
\begin{align}
    \label{eq:sub-exp}
    \log\left(\EX [e^{-\lambda (H^2 - \mathbb{E}(H^2))}]\right) \leq 16\lambda^2 \sigma_w^2 + 8\lambda^2 \sigma_e^2.
\end{align}
Thus, Theorem \ref{thm:posterior_Samlping} can be used with $$b = \frac{1}{32\sigma_w^2 + 16 \sigma_e^2}$$ and $\varphi(\lambda) = 16\lambda^2 \sigma_w^2 + 8\lambda^2 \sigma_e^2$. The Legendre dual can be calculated as
\begin{equation*}
    \varphi^*(\gamma) = 
    \begin{cases}
    \frac{\gamma^2}{64\sigma_w^2 + 32 \sigma_e^2} & \gamma \leq 1\\
    \frac{\gamma}{32\sigma_w^2 + 16\sigma_e^2} - \frac{1}{64\sigma_w^2 + 32\sigma_e^2} & \gamma \geq 1
    \end{cases}
\end{equation*}
which has an inverse function equal to
\begin{equation}
    \label{eq:phi_star_inverse}
    \varphi^{*-1}(\gamma) = 
    \begin{cases}
    \sqrt{(64\sigma_w^2 + 32\sigma_e^2) \gamma} & \gamma \leq \frac{1}{64\sigma_w^2 + 32\sigma_e^2}\\
    (32\sigma_w^2 + 16\sigma_e^2) \gamma + \frac{1}{2} & \gamma \geq \frac{1}{64\sigma_w^2 + 32\sigma_e^2}
    \end{cases}
\end{equation}
Hence, by the application of Theorem \ref{thm:posterior_Samlping}, we have
\begin{align*}
\sup_{P_W\in\mathcal{M}_W} &\EX_{W\sim P_W} [e(\alg_{\text{PS}},W)] \\
&\le \sup_{P_W\in\mathcal{M}_W} \varphi^{*-1} \Big(\frac{1}{n}\Big(  I(W;Z^n) + \frac{c}{2\sigma_w^2}\Big)\Big).
\end{align*}
Based on Theorem \ref{RLS:MI}, for any $P_W$ (i.e., any choice of $\mu$) as $n\to \infty$, we have $I(W;Z^n) = O(d\log(n))$ which proves the Theorem.
\end{proof}

\IEEEpeerreviewmaketitle

\section{Introduction}
\section*{Acknowledgment}
The authors would like to thank the anonymous reviewers for their valuable comments. 

\ifCLASSOPTIONcaptionsoff
  \newpage
\fi

\bibliographystyle{IEEEtran}
\bibliography{main.bib}

\vfill\eject

\begin{IEEEbiographynophoto}{Hassan Hafez-Kolahi}
received the B.Sc. degree in Computer Engineering from Ferdowsi University, Mashhad, Iran in 2011. He received his M.Sc.  degree in 2013, and his  Ph.D. degree in 2022 from Sharif University of Technology, Tehran, Iran. His research interests are in learning theory,  information theory and interactive proof systems.\end{IEEEbiographynophoto}

\begin{IEEEbiographynophoto}{Behrad Moniri} received the B.Sc. degree in electrical engineering from Sharif University of Technology, Tehran, Iran in 2020 with highest distinctions.  He is currently pursuing the  Ph.D. degree in electrical and systems engineering with the University of Pennsylvania, Philadelphia, PA, USA.  His research interests include deep learning theory, high-dimensional asymptotics, and information theory. 
\end{IEEEbiographynophoto}

\begin{IEEEbiographynophoto}{Shohreh Kasaei} (M’05–SM’07) received the B.Sc. degree
from the Department of Electrical and Computer
Engineering (ECE), Isfahan University of Technology, Iran, in 1986. She then received the M.Sc. degree from Graduate School of Engineering and Science, Department of
Electrical and Electronics Engineering, University of the
Ryukyus, Japan, in 1994, and the Ph.D. degree from Signal
Processing Research Centre, School of Electrical
Engineering and Computer Science (EECS), Queensland
University of Technology (QUT), Australia, in 1998.
She was awarded as the best graduate student in engineering faculties of University of the Ryukyus, in 1994, the best Ph.D. student studied in overseas by the ministry of Science, Research, and Technology of Iran, in 1998, and as a distinguished researcher of Sharif University of Technology , in 2002 and 2010, where she is currently a full professor. She is the director of Image Processing Lab (IPL) since 1999. Her research interests are in Machine Learning (Computer Vision and Image/Video Processing) with primary emphasis on: 3D deep visual tracking, 3D semantic segmentation, 3D adversarial attack and defense, point cloud classification, adversarial knowledge distillation, autonomous driving cars, virtual/augmented
reality, dynamic 3D pose estimation, dynamic 3D action recognition, multi-resolution texture analysis, scalable video coding, image retrieval, video indexing, face recognition, hyperspectral change detection, video restoration, and fingerprint authentication.
\end{IEEEbiographynophoto}

\vfill

\end{document}